\documentclass[11pt]{article}

\usepackage[
  shownumpages,
  bgcolor={250,248,241},
  braincolor={157,225,252},
  linkcolor={133,76,101},
  citecolor={133,76,101},
  urlcolor={133,76,101},
  citingstyle=authoryear,
  bibliostyle=unsrtnat,
  bibfile=references
]{Styles/brainlab}

\usepackage{placeins}  % \FloatBarrier keeps the main-text tables above the references
\tcbuselibrary{breakable}
\tcbset{brainstyle/.append style={breakable}}
\usepackage{microtype}
\usepackage{parskip}   % the template indents nothing, so paragraphs need a vertical gap
\usepackage{multirow}
\usepackage{enumitem}
\usepackage{wrapfig}
\usepackage{tikz}
\usetikzlibrary{positioning,calc}

\definecolor{signblue}{RGB}{100, 140, 180}
\definecolor{muonred}{RGB}{180, 110, 100}
\definecolor{intersect}{RGB}{130, 95, 158}
\definecolor{textdark}{RGB}{25, 25, 25}
\definecolor{good}{RGB}{200,231,200}
\definecolor{bad}{RGB}{247,212,212}
\newcommand{\cg}[1]{\cellcolor{good!#1!bad}}

\newcommand{\EE}{\mathbb{E}}
\newcommand{\msign}{\mathrm{msign}}

\DeclareBrainTcbTheorem{corollary}{Corollary}
\DeclareBrainTcbTheorem{lemma}{Lemma}
\DeclareBrainTcbTheorem{proposition}{Proposition}

\let\STATE\State
\let\IF\If
\let\ELSE\Else
\let\ENDIF\EndIf
\let\FOR\For
\let\ENDFOR\EndFor
\providecommand{\REQUIRE}{\State \textbf{Require:} }

\setbrainmeta{
  title={LionMuon: Alternating Spectral and Sign Descent for Efficient Training},
  authors={
    Arman Bolatov\textsuperscript{\textdagger,1},
    Artem Riabinin\textsuperscript{\textdagger, 2, 3},
    Nikita Kornilov\textsuperscript{\textdagger, 2, 3, 4},
    Andrey Veprikov\textsuperscript{\textdagger, 2, 3, 5},
    Samuel Horv\'ath\textsuperscript{1},
    Martin Tak\'a\v{c}\textsuperscript{1},
    Aleksandr Beznosikov\textsuperscript{2, 3, 6}
  },
  affiliations={
    \textsuperscript{1}Mohamed bin Zayed University of Artificial Intelligence (MBZUAI)\\
    \textsuperscript{2}Basic Research of Artificial Intelligence Laboratory (BRAIn Lab)\\
    \textsuperscript{3}Moscow Independent Research Institute of Artificial Intelligence (MIRAI)\\
    \textsuperscript{4}Applied Artificial Intelligence Institute  \\
    \textsuperscript{5}SB AI Lab  \\
    \textsuperscript{6}Innopolis University  \\
    \textsuperscript{\textdagger}Equal contribution
  },
  abstract={
    \texttt{Muon}'s spectral step gives a stronger direction than a sign step, but it is expensive. Every step runs Newton--Schulz iterations on the full matrix and, in distributed training, an extra all-reduce. Sign steps, as in \texttt{Lion} and \texttt{Signum}, are cheap and stay local to each device. We propose \texttt{LionMuon}, which takes one \texttt{Muon} step every $P$ iterations and \texttt{Lion} steps in between, with a single dual-EMA momentum buffer shared by both. \texttt{Muon}'s compute and communication are paid once per $P$ steps, and the optimizer state is one buffer per matrix, half of \texttt{AdamW}'s. A single-EMA variant, \texttt{SignMuon}, isolates the effect of alternation. We prove complexity bounds under heavy-tailed noise in which the period interpolates between the smoothness and noise constants of \texttt{Muon} and \texttt{Lion}, and which state, through measurable ratios of these constants, when \texttt{LionMuon}'s complexity is below \texttt{Muon}'s. On 124M and 355M models trained on FineWeb, \texttt{LionMuon} with $P{=}2$ and $P{=}5$ reaches a lower loss than \texttt{Muon}, \texttt{AdamW}, \texttt{Lion} and \texttt{Signum} at the same number of tokens. Under 4-GPU data-parallel training \texttt{LionMuon} reaches \texttt{Muon}'s final loss with up to $30\%$ less wall-clock on PCIe, and it beats the communication-efficient \texttt{Muon} variants \texttt{Dion} and \texttt{MuonBP} on loss at no more exposed communication, while keeping the exact gradient. \\
Code: \url{https://github.com/brain-lab-research/lion-muon}.
  },
}

\begin{document}
\begin{mainpart}

%==============================================================================
\section{Introduction}
\label{sec:intro}
%==============================================================================

Training Large Language Models (LLMs) is a billion-parameter, million-step optimization problem in which per-step cost determines the final compute bill \citep{hoffmann2022chinchilla, kimi2025k2}. Update rules that are cheap per step and converge fast therefore cut that bill directly \citep{dahl2023algoperf, kasimbeg2025algoperf}.

A useful way to organize this design space is the Linear Minimization Oracle (LMO) viewpoint, which originates from Frank-Wolfe optimization \citep{jaggi2013fw}. Recent work reinterprets a wide family of first-order optimizers as norm-constrained linear oracles \citep{chen2023lion, veprikov2025preconditioned}. In particular, the parameter update at step $t$ takes the form:
\begin{eqnarray}
W_{t+1} \;=\; W_t + \eta_t\,\mathrm{LMO}_{\|\cdot\|}(\hat{G}_t),
\qquad
\mathrm{LMO}_{\|\cdot\|}(G) := \arg\min{_{\|S\| \leq 1} }\langle G, S \rangle,  \label{eq: LMO step no wd}   
\end{eqnarray}
where $\hat{G}_t$ is a momentum-smoothed gradient, $\|\cdot\|$ is a chosen norm and $\eta_t > 0$ is the learning rate, usually with warm-up and decay~\citep{goyal2017accurate, loshchilov2017sgdr, riabinin2026doeswarmupcomefrom}.
The choice of norm picks the optimizer: Frobenius norm $\|\cdot\|_F$ gives \texttt{normalized SGD}~\citep{hazan2015normsgd}, $\|\cdot\|_\infty$ gives \texttt{signSGD} and its momentum variants \texttt{Signum}~\citep{bernstein2018signsgd} and \texttt{Lion}~\citep{chen2024lion}, and the spectral norm $\|\cdot\|_2$ gives \texttt{Muon}~\citep{jordan2024muon}. These methods now drive production LLM training, with \texttt{Muon} and its variants powering Moonlight \citep{liu2025moonlight}, Kimi K2 \citep{kimi2025k2}, and DeepSeek V4 \citep{deepseekai2026deepseekv4}.

\begin{wrapfigure}[14]{r}{0.27\textwidth}
\centering
\resizebox{0.265\textwidth}{!}{%
\begin{tikzpicture}[font=\sffamily]
  \fill[signblue, opacity=0.42] (-2.0, 0) ellipse (4.2 and 3.8);
  \fill[muonred,  opacity=0.42] ( 2.0, 0) ellipse (4.2 and 3.8);
  \begin{scope}
    \clip (-2.0, 0) ellipse (4.2 and 3.8);
    \fill[intersect, opacity=0.60] (2.0, 0) ellipse (4.2 and 3.8);
  \end{scope}
  \draw[signblue!65!black, line width=1.2pt] (-2.0, 0) ellipse (4.2 and 3.8);
  \draw[muonred!65!black,  line width=1.2pt] ( 2.0, 0) ellipse (4.2 and 3.8);

  \node[align=center, text=textdark, font=\sffamily\bfseries\large]
    at (-3.7, 1.5) {Sign-based steps\\[1pt]{\normalsize(Lion / Signum)}};

  \node[align=center, text=textdark, font=\sffamily\bfseries\large]
    at (3.6, 1.5) {Muon\\[1pt]{\normalsize Optimizer}};

  \node[align=left, text=textdark, font=\sffamily\small]
    at (-4.05, -0.9) {%
      $\bullet$\ Weaker update quality\\[7pt]
      $\bullet$\ Low compute cost\\[2pt]
      \quad(cheap sign updates)};

  \node[align=left, text=textdark, font=\sffamily\small]
    at (4.0, -0.9) {%
      $\bullet$\ Stronger update quality\\[7pt]
      $\bullet$\ Extra compute and\\[2pt]
      \quad communication cost};

  \node[align=center, text=white, font=\sffamily\bfseries\large]
    at (0, 1.0) {Our Methods};

  \node[align=left, text=white, font=\sffamily\small]
    at (-0.1, -0.8) {%
      $\bullet$\ Lowest loss in our runs\\[5pt]
      $\bullet$\ Lower cost than Muon\\[5pt]
      $\bullet$\ Recover sign and Muon\\[2pt]
      \quad limiting cases};
\end{tikzpicture}}
\caption{Sign steps are cheap, \texttt{Muon} steps are strong but expensive, and our methods alternate between the two.}
\label{fig:lionmuon-overview}
\end{wrapfigure}
Within this family, sign-based methods sit at the cheap end: \texttt{Signum} updates with the sign of a single momentum buffer, while \texttt{Lion} uses two EMA timescales but keeps the same coordinate-wise sign step. 
\texttt{Muon} sits at the opposite end, computing the matrix sign $\msign(X)=UV^\top$ via Newton--Schulz iterations~\citep{bernstein2024old}. The resulting spectral direction is often much stronger than a coordinate-wise sign step~\citep{chen2025muon}, but each \texttt{Muon} step costs several Newton--Schulz iterations of matrix multiplications.

Distributed training infrastructure is built around element-wise optimizers: every device holds a copy or a shard of each matrix and updates it with a local rule that needs neither the whole matrix nor any communication. \texttt{Muon}'s update couples the entire matrix, so a sharded matrix has to be gathered before Newton--Schulz~\citep{muonbp2025, essential2025layersharding}, and under data parallelism the Newton--Schulz work is either repeated on every device, or dealt out across them and exchanged afterwards in a synchronization that cannot hide behind the backward pass~\citep{liu2025moonlight, chen2026dmuon}. The reported price ranges from a $5$ to $10\%$ throughput loss under tensor parallelism~\citep{muonbp2025} to more than twice the cost of forward and backward in a naive implementation~\citep{chen2026dmuon}, and careful implementations shrink it but pay it on every step. A sign step is element-wise and needs none of it (Figure~\ref{fig:lionmuon-overview}), which raises the question:
\begin{quote}
    {\emph{Can we keep the update quality of \texttt{Muon} while paying its compute and communication cost only once every $P$ steps?}}
\end{quote}
\vspace{6pt}
Starting from \texttt{Signum}, inserting one \texttt{Muon} step every $P$ iterations gives \texttt{SignMuon}. Replacing its momentum with \texttt{Lion}'s dual-EMA rule then yields \texttt{LionMuon}, the main method we study.

\paragraph{Contributions.}
\begin{itemize}
    \item \texttt{SignMuon} and \texttt{LionMuon} (Section~\ref{sec:algorithm}) add one integer $P$ to \texttt{Lion} and have \texttt{Muon}, \texttt{Lion} and \texttt{Signum} as special cases. The Newton--Schulz compute and the communication of \texttt{Muon} are paid once per $P$ steps, the sign steps are local, and the state is one buffer (Appendix~\ref{sec:cost}).
    \item Complexity bounds under heavy-tailed noise, with and without weight decay (Section~\ref{sec: conv bounds}, Appendix~\ref{app: weight decay}). The period sets an interpolation between \texttt{Muon}'s and \texttt{Lion}'s smoothness and noise constants, and the bound says when \texttt{LionMuon}'s complexity is below \texttt{Muon}'s, a condition we check against constants measured during training (Section~\ref{sec: theory discussion}).
    \item Experiments at 124M on FineWeb~\citep{penedo2024fineweb} and WikiText-103~\citep{merity2017wikitext}, tuned per method with three seeds, and at 355M on FineWeb with transferred hyperparameters (Section~\ref{sec:results}). \texttt{LionMuon} with $P{\in}\{2,5\}$ beats \texttt{Muon}, \texttt{AdamW}, \texttt{Lion} and \texttt{Signum} at both sizes. Under 4-GPU data-parallel training we count the bytes each step sends, time the step on PCIe and NVLink, and compare with \texttt{Dion}~\citep{ahn2025dion} and \texttt{MuonBP}~\citep{muonbp2025}.
\end{itemize}

%==============================================================================
\section{Related work}
\label{sec:related}
%==============================================================================

\textbf{Sign-based methods.}
Sign-based methods first appeared as a communication-efficient solution for distributed optimization~\citep{bernstein2018signsgdmajority}. The element-wise sign update  $
W_{t+1} = W_t - \eta_t\,\mathrm{sign}(\hat{G}_t)  
$  is cheap to compute, to parallelize and to transmit. Sign methods are also valued in LLM training for their memory efficiency, for zeroth-order fine-tuning \citep{petrov2025leveraging}, and for their robustness to heavy noise \citep{kornilov2025sign, yu2026sign} and to complex models \citep{crawshaw2022robustness}.

\textbf{Spectral methods.}
\texttt{Muon}~\citep{jordan2024muon} moves from element-wise sign to its matrix analogue computed by Newton--Schulz (NS) iterations:
$
W_{t+1} = W_t - \eta_t\,\mathrm{NS}_{K}(\hat{G}_t).
$
The \texttt{MuonClip} variant powers Kimi K2 \citep{kimi2025k2}. Other works refine \texttt{Muon} itself, e.g., \texttt{Gluon} \citep{riabinin2025gluon} and \texttt{HTMuon} \citep{pang2026htmuon}. A second line makes it cheaper to run at scale. \texttt{Dion} \citep{ahn2025dion} replaces Newton--Schulz by a power iteration that keeps a low-rank factorization of the momentum with error feedback, and synchronizes the two factors instead of the gradient, so what is sent grows with the rank rather than with the matrix. \texttt{MuonBP} \citep{muonbp2025} orthogonalizes each tensor-parallel shard on its own device and gathers the full matrix only every $P$-th step, which removes the all-gather from most steps at the price of a block-wise update in between. Layer sharding \citep{essential2025layersharding} keeps whole matrices on one device so that Newton--Schulz needs no gather. \texttt{DMuon} \citep{chen2026dmuon} keeps the update exactly and removes the redundant work of naive distributed implementations: each matrix is assigned to one owner rank and Newton--Schulz runs in its Gram form, which brings \texttt{Muon}'s step time to within a few percent of \texttt{AdamW}'s on 8 to 256 GPUs.

Two concurrent methods target \texttt{Muon}'s per-step cost directly. \texttt{LiMuon} \citep{huang2025limuon} replaces Newton--Schulz with a low-rank randomized SVD of the momentum, and \texttt{OLion} \citep{wang2026olion} composes orthogonalization with an element-wise sign inside every step. We instead keep the spectral oracle as a black box and call it only every $P$ iterations, so either of them, or a faster polynomial iteration for the matrix sign~\citep{bernstein2024old, amsel2025polarexpress, grishina2025accelerating}, can be dropped in and the savings compound. The same holds for implementations such as \texttt{DMuon}: they make each \texttt{Muon} step cheaper, we make them rarer.

\textbf{\texttt{Lion}-$\mathcal{K}$ framework.}
\citet{chen2023lion} view \texttt{Lion} and \texttt{signSGD} as solving a constrained problem with $\mathcal{K}=\|\cdot\|_1$ through a Lyapunov analysis, and the same machinery covers \texttt{Muon} with $\mathcal{K}=\|\cdot\|_{\mathrm{nuc}}$ \citep{chen2025muon}. The stochastic Frank-Wolfe view \citep{sfyraki2025lions} recovers the rates of both families in one language. Our analysis builds on this common basis.

\textbf{Optimizer switching.}
Combining different optimizers within a single run is an established idea: \texttt{SWATS} \citep{keskar2017improving} switches from \texttt{Adam} to \texttt{SGD}, \texttt{AdaBound} \citep{luo2019adabound} interpolates between them through learning-rate clipping, and \texttt{AGD} \citep{yue2023agd} gates between the two adaptively. To our knowledge, no prior work studies the periodic switching between \texttt{Muon} and \texttt{Lion}-style steps that we propose here.

%==============================================================================
%==============================================================================

\section{Algorithm}
\label{sec:algorithm}

\subsection{Notation}

We work in the matrix parameter space $\mathbb{R}^{m \times n}$ and denote the parameter matrix at iteration $t$ by $W_t \in \mathbb{R}^{m \times n}$  and its stochastic gradient by $G_t \in \mathbb{R}^{m \times n}$. This space is equipped with the Frobenius inner product $\langle X, Y \rangle := \operatorname{tr}(X^\top Y), \|X\|_F^2 = \langle X, X \rangle$ and with the following matrix norms:
\begin{align*}
&\|X\|_2 := \sigma_1, \quad
\|X\|_\infty := \max_{ij} |X_{ij}|, \quad \|X\|_{\mathrm{nuc}} := \sum_k \sigma_k, \quad \|X\|_1 := \sum_{ij} |X_{ij}|, 
\end{align*}
where $\sigma_1 \ge \dots \ge \sigma_{\min(m,n)}$ are the sorted singular values of matrix $X \in \mathbb{R}^{m \times n}$.
The dual norm $\|X\|_\star := \sup_{\|S\| \le 1} \langle X, S \rangle$ gives dual pairs $\|\cdot\|_{2,\star} = \|\cdot\|_{\mathrm{nuc}}$ and $\|\cdot\|_{\infty,\star} = \|\cdot\|_1$.
For all matrices $X\in \mathbb{R}^{m \times n}$, the considered norms satisfy the following inequalities: 
\begin{eqnarray}
    \|X\|_\infty \le \|X\|_2 \le \|X\|_F \le \sqrt{mn}\,\|X\|_\infty \quad \text{and} \quad \tfrac{1}{\sqrt{mn}}\|X\|_1 \le \|X\|_F \le \|X\|_{\mathrm{nuc}} \le \|X\|_1. \label{eq: norm relations}
\end{eqnarray}
We use the spectral norm LMO to calculate the matrix-sign operation $\mathrm{LMO}_{\|\cdot\|_2}(G) = - \msign(G)$ and the infinity norm LMO to calculate the element-wise sign $\mathrm{LMO}_{\|\cdot\|_\infty}(G) = - \mathrm{sign}(G)$.

%==============================================================================
\subsection{\texttt{LionMuon}}
%==============================================================================

Algorithm~\ref{alg:lionmuon} keeps a single momentum buffer $M_t$, updated at every step. Each iteration forms the direction $\hat{G}_t$ as \texttt{Lion} does, by interpolating between $M_{t-1}$ and the current gradient $G_t$. Every $P$-th iteration takes a \texttt{Muon} step, which orthogonalizes $\hat{G}_t$ with Newton--Schulz, and every other iteration takes an element-wise sign step. Both use decoupled weight decay $\lambda$.

\begin{algorithm}{\texttt{LionMuon} and \texttt{SignMuon} for a single 2D parameter $W \in \mathbb{R}^{m \times n}$}
\label{alg:lionmuon}
\begin{algorithmic}[1]
\REQUIRE Horizon $T$, period $P \in \{1,2,\ldots\}\cup\{\infty\}$ ($P{=}\infty$ means the \texttt{Muon} branch is never taken), learning rates $\eta_M$ (\texttt{Muon}) and $\eta_L$ (\texttt{Lion}), betas $\beta_1, \beta_2 \in [0,1)$, weight decay $\lambda \ge 0$, NS steps $K_{\mathrm{NS}}$, initial parameters $W_0$ and momentum $M_{-1} = 0$, and $c(A) := 0.2\sqrt{\max(\text{rows}(A), \text{cols}(A))}$.
\FOR{$t = 0, 1, \ldots, T-1$}
  \STATE $G_t = \nabla_W \mathcal{L}_t$ \hfill $\triangleright$ Stochastic gradient
  \STATE $\hat{G}_t = \beta_1 M_{t-1} + (1 - \beta_1) G_t$ \hfill $\triangleright$ \texttt{Lion} interpolation (direction) %%\footnote{The interpolation uses the buffer $M_{t-1}$ \emph{before} this step's momentum update on line~\ref{line:lionmuon-mom}. Some \texttt{Lion} pseudocode in the literature uses the post-update buffer $M_t$ instead, which expands to $\beta_1\beta_2 M_{t-1} + [\beta_1(1-\beta_2) + (1-\beta_1)]\,G_t$; this is a \emph{different} update rule (different effective $\beta$'s), not just a relabelling. We use the pre-update form throughout, matching our reference implementation.}
  \IF{$t \bmod P = 0$}
    \STATE \label{line:lionmuon-muon} $W_{t+1} = W_t - \eta_M \, \bigl(c(\hat{G}_t)\,\mathrm{NS}_{K_{\mathrm{NS}}}(\hat{G}_t) + \lambda W_t\bigr)$ \hfill $\triangleright$ \texttt{Muon} step
  \ELSE
    \STATE \label{line:lionmuon-lion} $W_{t+1} = W_t - \eta_L \, \bigl(\mathrm{sign}(\hat{G}_t) + \lambda W_t\bigr)$ \hfill $\triangleright$ \texttt{Lion} step
  \ENDIF
  \STATE \label{line:lionmuon-mom} $M_t = \beta_2 M_{t-1} + (1 - \beta_2) G_t$ \hfill $\triangleright$ Momentum update (every step)
\ENDFOR
\end{algorithmic}
\end{algorithm}

\paragraph{Implementation notes.}
\texttt{LionMuon} keeps one buffer $M_t$ per matrix, since $\hat{G}_t$ is computed in place, so per matrix its state matches \texttt{Lion} and \texttt{Muon} and is half of \texttt{AdamW}'s. The weight matrices of the transformer blocks and the positional embedding take the \texttt{LionMuon} update. Following the usual \texttt{Muon} convention \citep{jordan2024muon}, the token embedding, which is tied to the output head, and the 1D parameters (biases and norm gains) use \texttt{AdamW} at a fixed $10^{-3}$, and they do so under every method we compare, so the comparison isolates the update on the matrices. Giving them the \texttt{Lion} step instead is clearly worse (Appendix~\ref{app:ablate}). Table~\ref{tab:special-cases} lists the special cases. 

%==============================================================================
\section{Convergence analysis}
\label{sec:convergence}
%==============================================================================

The main text analyzes \texttt{LionMuon} (Algorithm~\ref{alg:lionmuon}) without weight decay ($\lambda = 0$), and Appendix~\ref{app: weight decay} covers weight decay, with the same conclusions. Appendix~\ref{app: missing proofs} proves Theorem~\ref{thm: main_convergence lionmuon no wd} and Corollary~\ref{col: optimal params limuon no wd}.

\subsection{Assumptions}\label{sec: ass}

\begin{assumption}[Smoothness and lower boundness]
\label{assum:smoothness}
The objective function $f : \mathbb{R}^{m \times n} \to \mathbb{R}$ is lower bounded by $f_\star$ and  $L$-smooth with respect to a primal norm $\|\cdot\|$:
\[
\|\nabla f(W) - \nabla f(W')\|_\star \le L\,\|W - W'\|, \quad \text{for all } W, W' \in \mathbb{R}^{m \times n}.
\]
We use smoothness constants $L_2$ and $L_\infty$ for norms $\|\cdot\|_2$ and $\|\cdot\|_\infty$, respectively.
\end{assumption}

From the norm inequalities \eqref{eq: norm relations}, we can bound the smoothness ratio $1 \le L_\infty / L_2 \le mn$. Following \citet{sadiev2023high}, \citet{hubler2024gradient} and \citet{chezhegov2026high}, we allow heavy-tailed noise, which has been reported in LLM training \citep{gurbuzbalaban2021heavy}. 
\begin{assumption}[Bounded $\kappa$-th moment]
\label{assum:variance}
Conditionally on the past, stochastic gradients $G_t$ are unbiased estimates of the true gradient $\nabla f(W_t)$ and have bounded $\kappa$-th moment for some $\kappa \in (1, 2]$ and $\sigma \geq 0$:
\[\EE[G_t \mid W_t] = \nabla f(W_t), \quad \quad
\EE\bigl[\|G_t - \nabla f(W_t)\|_F^{\kappa} \mid W_t\bigr] \le \sigma^{\kappa}.
\]
\end{assumption}
A Hill estimator on the norms of the gradient noise, at four checkpoints (steps $0$ to $3{,}000$) of a 124M \texttt{LionMuon} $P{=}2$ run on WikiText-103, gives tail indices between $34$ and $68$ at batch $32$, between $21$ and $45$ at batch $8$, and between $10.6$ and $23$ at batch $2$. All of them are above $2$, so Assumption~\ref{assum:variance} holds with $\kappa = 2$, and the tails thicken as the batch shrinks.
The noise level also depends on the dual norm, which we bound in expectation following \citet{kornilov2023accelerated} and \citet{hubler2024gradient}.
\begin{assumption}[Noise norm equivalence]
\label{assum:norm_eq}
For any linear combination $\sum_\tau a_\tau \epsilon_\tau$ of the gradient noise terms $\epsilon_\tau := G_\tau - \nabla f(W_\tau)$, which form a martingale difference sequence under Assumption~\ref{assum:variance}, we have:
\[
\EE\bigl[\|{\textstyle\sum_\tau} a_\tau \epsilon_\tau\|_\star\bigr] \le \rho_\star \cdot \EE\bigl[\|{\textstyle\sum_\tau} a_\tau \epsilon_\tau\|_F\bigr] \quad \text{for some  level } \rho_\star > 0.
\]
We use noise levels $\rho_{\mathrm{nuc}}$ and $\rho_1$ for dual norms $\|\cdot\|_{\mathrm{nuc}}$ and $\|\cdot\|_1$, respectively.
\end{assumption}
Since $\|X\|_{\mathrm{nuc}} \le \sqrt{\min\{m,n\}}\,\|X\|_F$ and by \eqref{eq: norm relations}, the levels obey the per-matrix bounds $\rho_{\mathrm{nuc}} \leq \sqrt{\min\{m,n\}}$ and $\rho_1 \leq \sqrt{mn}$, and the measured levels lie below them, $\rho_{\mathrm{nuc}}$ at most half of its bound (Table~\ref{tab:constants}).

\subsection{Convergence bound} \label{sec: conv bounds}

The analysis takes the exact matrix sign at unit scale: the Newton--Schulz approximation and the scale $c(\hat{G}_t)$ of Algorithm~\ref{alg:lionmuon} are not modeled, and the scale only rescales $\eta_M$ (and, with weight decay, $\lambda$) on the \texttt{Muon} steps.

\begin{theorem}[Convergence bound of \texttt{LionMuon}]
\label{thm: main_convergence lionmuon no wd}
Let the objective function $f$ satisfy Assumption \ref{assum:smoothness} with respect to $\|\cdot\|_2$ with constant $L_2$, and with respect to $\|\cdot\|_{\infty}$ with constant $L_{\infty}$. Let noise Assumptions \ref{assum:variance} and \ref{assum:norm_eq} hold with noise constants $\sigma$, $\rho_\text{nuc}$ and $\rho_{1}$. Fix a horizon $T$, period $P \in [1, \infty]$, momentum parameters $\beta_1 \in [0, 1)$, $\beta_2 \in (0, 1)$ and learning rates $\eta_{M}$ and $\eta_{L}$. 

Define the period-averaged learning rate, noise level and smoothness:
\begin{eqnarray}
\bar{\eta} := \tfrac{\eta_{M}}{P} + \tfrac{(P-1) \eta_{L}}{P},
\quad
\bar{\rho} := \tfrac{\eta_M}{P \bar{\eta}} \rho_{\text{nuc}} + \tfrac{(P-1)\eta_L}{P \bar{\eta}} \rho_{1},
\quad
\bar{L} := \tfrac{\eta_M \tilde{\eta}_{\max}}{P \bar{\eta}^2} L_2 + \tfrac{(P-1)\eta_L \eta_{\max}}{P \bar{\eta}^2} L_\infty, \label{eq: period avr constants}
\end{eqnarray}
where $\tilde{\eta}_{\max} = \max\{\eta_M, \sqrt{mn}\,\eta_L\}$ and $\eta_{\max} = \max\{\eta_M, \eta_L\}$ for intermediate $P \in (1, \infty)$, with the boundary cases $\tilde{\eta}_{\max} =\eta_{\max} = \eta_M$ at $P=1$ and $\tilde{\eta}_{\max} = \eta_{\max} = \eta_L$ at $P=\infty$.

Then, our \texttt{LionMuon} algorithm starting with $\Delta_0 := f(W_0) - f_\star, E_0 = \nabla f(W_0) - M_0$  guarantees the following bound on the period-averaged gradient dual norm:
\begin{align}
   \min_{i < \frac{T}{P}} \{\EE[\|\overline{\nabla} f(W_{i\cdot P})\|]\} &\leq \frac{\Delta_0}{ \bar{\eta}T }  +    \frac{4 \bar{L} \bar{\eta} }{(1-\beta_2)} +  \frac{4\beta_{1} }{\beta_2}\bar{\rho}  \sigma (1-\beta_2)^\frac{\kappa - 1}{\kappa} + 2 \left|1 - \frac{\beta_{1}}{\beta_2}\right| \bar{\rho}  \sigma \notag \\
   &\quad +  \frac{2 \beta_{1}}{\beta_2} \frac{\eta_{\max} \|E_0\|_{1}}{\bar{\eta}T (1 - \beta_2)},  \notag \\
     \text{where } \EE[\|\overline{\nabla} f(W_{i\cdot P})\|] &:= \frac{\left(\eta_M \cdot \EE[\|\nabla f(W_{i\cdot P})\|_{\text{nuc}}]  + \sum_{j = 1}^{P-1}[\eta_{L} \cdot \EE[\|\nabla f(W_{i\cdot P + j})\|_{1}] ] \right)}{\eta_M + (P-1)\eta_L }. \label{eq: minimal metric} %\geq \min_{j}\EE[\|\nabla f(W_{i\cdot P + j})\|_\text{nuc}], \notag
\end{align}

\end{theorem}

Prior analyses treat \texttt{Muon} or \texttt{Lion} on their own \citep{li2025note, shen2025convergence,an2025asgo, riabinin2025gluon}. Our bound covers both steps, with their different norms, and depends on the schedule only through the period-averaged constants \eqref{eq: period avr constants}. It has the optimal form for momentum-based norm-constrained methods under heavy-tailed noise \citep{liunonconvex, kornilov2025sign}, and its boundary cases recover the known bounds for \texttt{Muon} and \texttt{Lion} \citep{yu2026sign, nagashima2026improved, iiduka2026muon}. 

\begin{corollary}[Optimal Parameters for \texttt{LionMuon}] 
\label{col: optimal params limuon no wd}
Let the objective function $f$ and the noise satisfy Assumptions \ref{assum:smoothness}, \ref{assum:variance} and \ref{assum:norm_eq}  with the period-averaged constants $\bar{L}$, $\sigma$ and $\bar{\rho}$  defined in \eqref{eq: period avr constants}.

\begin{itemize}[leftmargin=15pt]
    \item  
    Fix a period $P \in (1, \infty)$ and the learning-rate ratio $\alpha = \eta_M/\eta_L$.

To achieve accuracy $\min_i \{\EE[\|\overline{\nabla} f(W_{i\cdot P})\|]\} \leq \varepsilon$ with $\varepsilon \leq 16\bar{L}\Delta_0/(P\|E_0\|_1)$, our \texttt{LionMuon} requires $T$ iterations:
\begin{equation}
    T = O\left(\bar{L}\Delta_0 \cdot \max \left\{\frac{( \bar{\rho}\sigma)^\frac{\kappa}{\kappa - 1} }{\varepsilon^\frac{3\kappa - 2}{\kappa - 1}}, \frac{1}{\varepsilon^2} \right\}\right),  \label{eq: T bound limuon no wd}
\end{equation}
with the optimal parameters:
\begin{gather*}
1 - \beta_2 = \min\left\{\left(\frac{\varepsilon}{32 \bar{\rho}\sigma}\right)^\frac{\kappa}{\kappa - 1}, \tfrac12 \right\}, \qquad \beta_1 \in \beta_2\left[ \max\left\{1 - \frac{\varepsilon}{16 \bar{\rho}\sigma}, 0\right\},1\right], \\
\eta_{L} = \frac{\varepsilon (1 - \beta_2)}{32 \left(\frac{\alpha}{P} + \frac{P-1}{P}\right) \cdot \bar{L}}.
\end{gather*}
\item Pure \texttt{Muon} ($P=1$) and \texttt{Lion}  ($P=\infty$) keep the same momentums $\beta_1, \beta_2$, number of iterations $T$ and only single learning rate $\eta_M = \frac{\varepsilon (1 - \beta_2)}{32  \cdot L_2}$ or $\eta_L = \frac{\varepsilon (1 - \beta_2)}{32  \cdot L_\infty}$ .
\item We can set single-EMA $\beta_1 = \beta_2$ to get optimal parameters for our \texttt{SignMuon}.
\end{itemize}

\end{corollary}
The complexity \eqref{eq: T bound limuon no wd} has the optimal dependence on $\varepsilon$: with the Frobenius constants $L_F$ and $\rho_F$ in place of $\bar{L}$ and $\bar{\rho}$ it is the known optimal rate \citep{zhang2020adaptivegood}. 

\paragraph{Tightening up the constants.} Because the analysis mixes \texttt{Lion} and \texttt{Muon} steps and their norms, it applies the worst-case norm inequalities \eqref{eq: norm relations}, so $\bar{L}$ in \eqref{eq: period avr constants} carries conservative factors such as $\eta_{\max}$ or $\sqrt{mn}$ that disappear in the pure regimes. In practice, gradients and updates in deep networks tend to be dense \citep{bernstein2018signsgd}, and ours are too (Table~\ref{tab:constants}). For these dense matrices, the norm inequalities usually yield the approximate equalities:
\begin{eqnarray}
    \|\nabla f(W_{t})\|_1 \approx \alpha \cdot \|\nabla f(W_{t})\|_\text{nuc}  \quad \text{for some large constant $1 \le \alpha \lesssim \sqrt{mn}$},\label{eq: dense constants eq}
\end{eqnarray}
and we can obtain a more natural interpolation $\bar{L} = \tfrac{\eta_M^2}{P \bar{\eta}^2} L_2 + \tfrac{(P-1)\eta_L^2}{P \bar{\eta}^2} L_\infty$ (Appendix~\ref{app: remark about ref smoothness}, under the matching approximation for the update matrices). 
\subsection{Discussion} \label{sec: theory discussion}

\paragraph{Choice of the learning-rate ratio.} The scale $\frac{\eta_M}{\eta_L}$ determines the interpolation between the smoothness constants, noise levels and gradient norms of \texttt{Muon} and \texttt{Lion}. By the norm relations \eqref{eq: norm relations}, the gradient norm $\|\cdot\|_\text{1}$ of \texttt{Lion} is never smaller than $\|\cdot\|_\text{nuc}$ of \texttt{Muon}, especially for dense matrices, so without scaling the \texttt{Lion} updates dominate the metric \eqref{eq: minimal metric} simply because they are large and frequent. We therefore choose a large scale $\eta_M/\eta_L = \alpha$ that brings the gradient norms \eqref{eq: dense constants eq} to the same order of magnitude: $\min_{i < \frac{T}{P}} \{\EE[\|\overline{\nabla} f(W_{i\cdot P})\|]\} \approx \frac{\alpha P \cdot \min_t \{\EE[\|\nabla f(W_t)\|_\text{nuc}]\}}{\alpha + (P-1)}$. The grid search over $(\eta_M, \eta_L)$ in Appendix \ref{app:heatmap} confirms that a large ratio is best for \texttt{LionMuon}, and that its tuned ratio grows with $P$ (Section~\ref{sec:results}). The theory cannot fix its value, so we tune the ratio and $\eta_M$.

\paragraph{Choice of period $P$.}

The reason to take $P > 1$ is cost. In the operation count of the bound a \texttt{Muon} step is $K_{\mathrm{NS}}$ operations, one per Newton--Schulz iteration, and a \texttt{Lion} step is one (Appendix~\ref{sec:cost} gives their share of a whole training step), and the \texttt{Muon} step is the one that communicates. Our complexity \eqref{eq: T bound limuon no wd} interpolates between the pure-\texttt{Muon} ($P{=}1$, $\bar{L} = L_2$, $\bar{\rho} = \rho_{\text{nuc}}$) and pure-\texttt{Lion} ($P{=}\infty$, $\bar{L} = L_\infty$, $\bar{\rho} = \rho_{1}$) regimes via the period-averaged smoothness $\bar{L} $ and noise $\bar{\rho}$. For typical dense gradients \eqref{eq: dense constants eq}, we set the ratio $\eta_M/\eta_L = \alpha$ to equalize the $\|\cdot\|_1$ and $\|\cdot\|_{\text{nuc}}$ gradient norms in the minimal metric \eqref{eq: minimal metric}. Then, for $P - 1 \ll \alpha$, the averaged learning rate $\bar{\eta}$ can be estimated by $\bar \eta \approx \eta_M/P$, and the refined averaged smoothness and noise become $\bar{L} \approx P^2\frac{L_2}{P} + P^2\frac{(P-1) }{ P}\frac{L_\infty}{\alpha^2}$ and $\bar{\rho} \approx P\frac{ \rho_{\text{nuc}}}{P} + P \frac{(P-1) }{ P}\frac{\rho_1}{\alpha}$.
We can now compare the number of operations $N$ that pure \texttt{Muon} and \texttt{LionMuon} need for the same accuracy $\min_t \{\EE[\|\nabla f(W_t)\|_\text{nuc}]\} \leq \varepsilon$. \texttt{LionMuon} does $\frac{K_\text{NS} + P - 1}{P}$ operations per iteration and only needs the weaker accuracy $P \varepsilon$ in \eqref{eq: T bound limuon no wd}, which gives $N = \phi(P)\, N_{\text{\texttt{Muon}}}$ with
\[
\phi(P) = \Bigl(\tfrac1P + \bigl(1-\tfrac1P\bigr)\tfrac{1}{K_\text{NS}}\Bigr)\Bigl(\tfrac1P + \bigl(1-\tfrac1P\bigr)\, r_L\Bigr)\Bigl(\tfrac1P + \bigl(1-\tfrac1P\bigr)\, r_\rho\Bigr)^{\frac{\kappa}{\kappa-1}},
\]
where $r_L := \frac{L_\infty}{\alpha^2 L_2}$, $r_\rho := \frac{\rho_1}{\alpha \rho_{\text{nuc}}}$ and $N_{\text{\texttt{Muon}}} = O\bigl[K_\text{NS} L_2\Delta_0 (\rho_{\text{nuc}}\sigma)^{\frac{\kappa}{\kappa-1}} / \varepsilon^{\frac{3\kappa-2}{\kappa-1}}\bigr]$.
Each bracket is a weighted mean of $1$ and a ratio, with weight $1/P$ on \texttt{Muon}'s end: the first is the cost of a step, the other two are the smoothness and the noise that the alternating method sees. \texttt{LionMuon}'s complexity is below \texttt{Muon}'s for every $P$ with $\phi(P) < 1$. When the three ratios $1/K_\text{NS}$, $r_L$ and $r_\rho$ are all below one this holds for every $P > 1$ (the estimate assumes $P - 1 \ll \alpha$), and $\phi$ keeps falling with $P$. When $r_L$ and $r_\rho$ are above one, the cheaper step competes with worse constants and $\phi$ can exceed one. In either case the bound compares each period with pure \texttt{Muon}, not the periods with one another: $\log \phi$ is concave in $1/P$, so its minimum over $[1, \infty]$ is at an endpoint, and which $P$ is best inside the family is left to experiment.
Table~\ref{tab:constants} measures the ratios at 124M: $r_L = 0.24$, $r_\rho = 0.98$ on FineWeb and $0.16$, $1.02$ on WikiText-103, with $\kappa = 2$ (Section~\ref{sec: ass}). The two constant brackets alone give $0.61$ at $P{=}2$ and $0.38$ at $P{=}5$ on FineWeb, so the bound puts these periods ahead of \texttt{Muon} in operations. Section~\ref{sec:results} finds \texttt{LionMuon} ahead of \texttt{Muon} at these periods, with a flat optimum between $P{=}2$ and $P{=}5$.

\paragraph{What the measurements say.} We measure the constants of the bound during 124M training (Appendix~\ref{app:constants}, Table~\ref{tab:constants} and Figure~\ref{fig:norm-diag}). The gradient ratio $\alpha$ settles near $67$ on FineWeb and $75$ on WikiText-103, inside the range of the tuned ratios $\eta_M/\eta_L$ of \texttt{LionMuon}, which grow from $10$ to $100$ with $P$ (Appendix~\ref{app:heatmap}), so the scale the theory asks for is of the order of the scale that trains best. On both datasets $r_L$ is well below one and $r_\rho$ about one, so the bound puts the alternating periods ahead of \texttt{Muon}.
%==============================================================================
\section{Experiments}
\label{sec:results}
%==============================================================================

\subsection{Setup}\label{sec:setup}

All runs use GPT-2-style decoders from the \texttt{llm-baselines} benchmark of \citet{semenov2025benchmark}, a cosine schedule and weight decay $0.1$.

\emph{Single GPU.} At 124M (12 layers, width 768, batch $32\times512$ tokens, $64{,}000$ steps, $1.05$B tokens) we train on FineWeb~\citep{penedo2024fineweb} and WikiText-103~\citep{merity2017wikitext} and compare \texttt{AdamW}~\citep{loshchilov2019adamw}, \texttt{Signum}~\citep{bernstein2018signsgd}, \texttt{Lion}~\citep{chen2024lion}, \texttt{Muon}~\citep{jordan2024muon}, and \texttt{SignMuon} and \texttt{LionMuon} with $P\in\{1,2,5,20\}$. Learning rates are tuned on full-length runs: $\eta_M$ over $10^{-4}$ to $10^{-1}$ and, for the alternating methods, the ratio $\eta_M/\eta_L$ over $1$ to $3000$ (Appendix~\ref{app:heatmap}). Momentum was tuned for \texttt{LionMuon} and \texttt{SignMuon} (Appendix~\ref{app:betas}). \texttt{Lion} and \texttt{Muon} use the same values, and \texttt{AdamW} the benchmark's. Each tuned setting is rerun with three seeds. At 355M (24 layers, width 1024, sequence 1024, batch 512, $15{,}650$ steps, $8.2$B tokens, above the Chinchilla budget~\citep{hoffmann2022chinchilla}) we train on FineWeb with every hyperparameter copied from the 124M optimum, again with three seeds, so this setting also tests how the tuned values transfer.

\emph{Four GPUs.} The 124M model is trained end to end with PyTorch DDP on four H200 GPUs (some repeated seeds on a second node), 8 sequences per GPU, for $150{,}000$ steps ($2.46$B tokens, the Chinchilla budget). Every hyperparameter comes from the 124M tuning, and the rates of \texttt{Muon}, \texttt{Lion} and \texttt{Dion} were rechecked at this horizon (Table~\ref{tab:ddp}). Every setting in Figure~\ref{fig:cost} is run with three seeds, and the four learning-rate checks in Table~\ref{tab:ddp} are single runs. Every GPU holds a full copy of the model, and the framework averages the gradients bucket by bucket while the backward pass is still running. \texttt{Muon} then orthogonalizes each matrix on one GPU and all-reduces the updates, as in Moonlight~\citep{liu2025moonlight}. \texttt{Dion}~\citep{ahn2025dion} and \texttt{MuonBP}~\citep{muonbp2025} use the settings of their papers and the learning rate of \texttt{Muon}. Both were designed for sharded models, so for \texttt{MuonBP} we emulate a four-way tensor-parallel split with four column blocks per matrix. Its block step is then local and its full step is \texttt{Muon}'s, so it differs from \texttt{SignMuon} only in what it does between full steps: a local \texttt{Muon} step on each block instead of a sign step. \texttt{Dion}'s low-rank gradient sync is the data-parallel mode its paper proposes. Appendix~\ref{app:setup} gives the details and pseudocode. Step times are measured on this node and on four A100 GPUs over PCIe, where communication is more expensive.

\subsection{Single-GPU training: 124M tuned, 355M transferred}
\label{sec:results-single}

\begin{figure}[t]
\centering
\includegraphics[width=\textwidth]{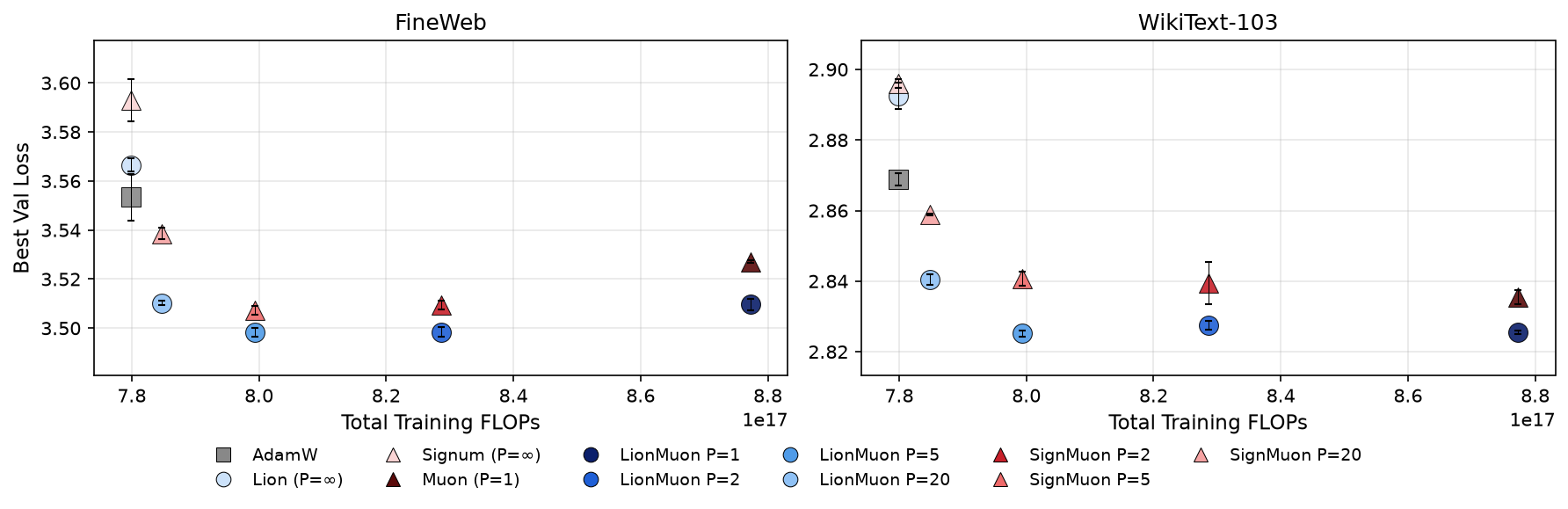}
\caption{Loss against training FLOPs at 124M on FineWeb (left) and WikiText-103 (right). Every optimizer is tuned. Mean over three seeds, with the standard deviation as error bars.}
\label{fig:124m}
\end{figure}

In the tuned 124M runs (Figure~\ref{fig:124m}), \texttt{LionMuon} with $P{\in}\{2,5\}$ reaches a lower loss than \texttt{Muon}, \texttt{AdamW}, \texttt{Lion} and \texttt{Signum} on both datasets, at fewer FLOPs than \texttt{Muon}, and its gap to \texttt{Muon} is larger than the seed spread. Table~\ref{tab:124m-seeds} lists the three-seed means and Appendix~\ref{app:heatmap} the tuned cells.

At 355M (Figure~\ref{fig:355m}), every alternating setting except \texttt{SignMuon} $P{=}5$ beats \texttt{Muon}, which beats \texttt{AdamW}, at fewer FLOPs. \texttt{Lion} lands between \texttt{Muon} and \texttt{AdamW} and \texttt{Signum} ends last, so at this size too alternating beats both of its endpoints (Table~\ref{tab:355m}). \texttt{LionMuon} $P{=}1$, which costs the same as \texttt{Muon}, also ends below it.

\textbf{Discussion.} Alternation and the dual EMA both help, and how much each contributes depends on the dataset: on FineWeb they contribute about equally, and on WikiText-103 the gain comes from the dual EMA. The loss of \texttt{LionMuon} is flat between $P{=}2$ and $P{=}5$ and rises at $P{=}20$, so $P{=}5$ gives the cheapest step at no cost in loss. For \texttt{LionMuon} the tuned sign-step rate hardly moves with $P$ while $\eta_M$ grows with it (Appendix~\ref{app:heatmap}), so when changing $P$, keep $\eta_L$ and scale $\eta_M$. The period does not have to be fixed either: a \texttt{Muon} step drawn with probability $1/P$ at every iteration gives the same loss (Appendix~\ref{app:ablate}).

\subsection{Data-parallel training on four GPUs}
\label{sec:results-distributed}

\begin{figure}[t]
\centering
\begin{minipage}[t]{0.48\textwidth}
\centering
\includegraphics[width=\textwidth]{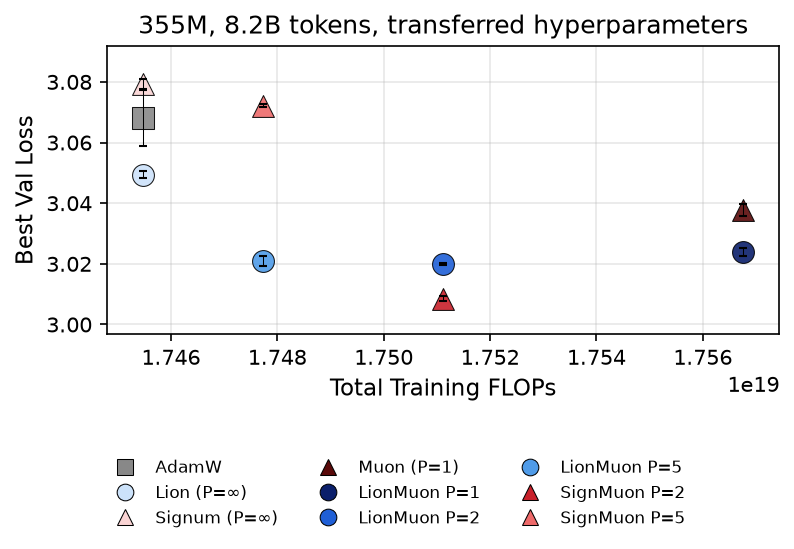}
\captionof{figure}{Loss against training FLOPs at 355M on FineWeb, with every hyperparameter copied from 124M. Mean over three seeds, with the standard deviation as error bars. For every method but \texttt{AdamW} the bar is smaller than the marker. Values in Table~\ref{tab:355m}.}
\label{fig:355m}
\end{minipage}\hfill
\begin{minipage}[t]{0.48\textwidth}
\centering
\includegraphics[width=\textwidth]{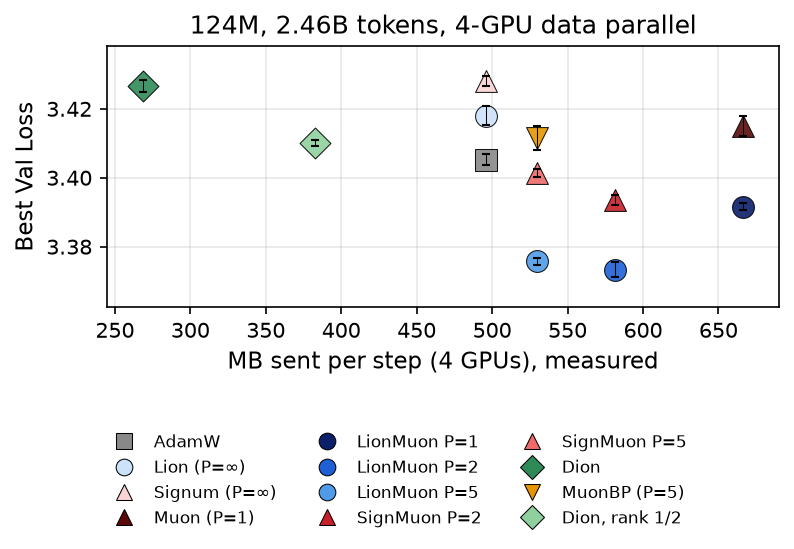}
\captionof{figure}{Loss against the bytes one step sends, for 124M trained on four GPUs with data parallelism. The bytes are counted from every collective the run issues. Baselines that had a rate sweep are shown at their best rate. Mean over three seeds, with the standard deviation as error bars. Values in Table~\ref{tab:ddp}.}
\label{fig:cost}
\end{minipage}
\end{figure}

Under data parallelism (Figure~\ref{fig:cost}), \texttt{LionMuon} $P{=}2$ and $P{=}5$ again reach the lowest loss, below \texttt{Muon}, the sign methods, \texttt{AdamW}, \texttt{Dion} and \texttt{MuonBP}. At this longer horizon \texttt{AdamW} ends below \texttt{Muon}, and \texttt{Lion} needs a rate three times lower than at $64$k steps to stay stable. The seed spread is at most $0.007$, a sixth of the gap between \texttt{LionMuon} and \texttt{Muon}, and \texttt{Muon} does no better at three times lower or higher rates (Table~\ref{tab:ddp}).

Per step, \texttt{LionMuon} $P{=}5$ sends a fifth less than \texttt{Muon} in total and a fifth as much exposed communication (Tables~\ref{tab:cost} and~\ref{tab:ddp}). \texttt{MuonBP} $P{=}5$ sends exactly the same amount, still runs Newton--Schulz on every step and ends level with \texttt{Muon}. In total bytes, \texttt{Dion} at rank $1/4$ (meaning $\min(m,n)/4$) sends the least and ends level with \texttt{Signum}, the two highest. At rank $1/2$ it sends more, still less than \texttt{AdamW}, and ends level with \texttt{Muon} and \texttt{MuonBP}, so its rank trades bytes for loss, and at neither rank does it come near \texttt{LionMuon}. \texttt{AdamW} sends slightly less than \texttt{LionMuon} $P{=}5$ and ends higher, so against \texttt{AdamW} the gain is in loss alone.

In step time (Table~\ref{tab:cost}), \texttt{Dion} and \texttt{MuonBP} are cheaper than \texttt{Muon} on PCIe, where the all-reduce is expensive, but slower on NVLink, since both trade bytes for extra arithmetic, while \texttt{LionMuon} is cheaper on both. Replicating the orthogonalization on every device would remove \texttt{Muon}'s all-reduce at the price of about a third more arithmetic (Appendix~\ref{sec:cost}). \texttt{LionMuon} divides both costs by $P$, so it is cheaper either way.

\begin{table}[t]
\centering
\caption{What one step costs on top of forward and backward, at 124M on four GPUs. Exposed communication is the optimizer's own all-reduce, which cannot start before the backward pass ends. Step times are relative to \texttt{Muon}, measured by running all methods back to back in a random order within each round and taking the ratio inside the round, so that any background load on the node falls on every method alike (30 rounds on PCIe, 40 on NVLink, 95\% intervals within $\pm0.02$ and $\pm0.09$). \texttt{Muon}'s median step took $1.56$\,s on PCIe and $0.24$\,s on NVLink. \texttt{Dion}'s bytes replace the gradient all-reduce. State is per 2D parameter $W$, and $r$ is \texttt{Dion}'s rank.}
\label{tab:cost}
\small
\begin{tabular}{@{}lccccc@{}}
\toprule
 & & Exposed & & \multicolumn{2}{c}{Step time / \texttt{Muon}} \\
\cmidrule(lr){5-6}
Optimizer & Newton--Schulz & MB/step & State & PCIe & NVLink \\
\midrule
\texttt{AdamW} & no & \cg{100} 0 & $2|W|$ & \cg{100} 0.75 & \cg{100} 0.85 \\
\texttt{Lion} / \texttt{Signum} & no & \cg{100} 0 & $|W|$ & -- & -- \\
\texttt{Muon} & every step & \cg{36} 171 & $|W|$ & \cg{0} 1.00 & \cg{77} 1.00 \\
\texttt{LionMuon} $P{=}2$ & every 2nd step & \cg{68} 85 & $|W|$ & \cg{48} 0.88 & \cg{89} 0.92 \\
\texttt{LionMuon} $P{=}5$ & every 5th step & \cg{87} 34 & $|W|$ & \cg{76} 0.81 & \cg{94} 0.89 \\
\texttt{MuonBP} $P{=}5$ & blocks each, full every 5th & \cg{87} 34 & $|W|$ & \cg{24} 0.94 & \cg{38} 1.25 \\
\texttt{Dion}, rank $\min(m,n)/4$ & no, QR on factors & \cg{0} 269 & $|W|+nr$ & \cg{96} 0.76 & \cg{0} 1.50 \\
\bottomrule
\end{tabular}
\end{table}

To reach \texttt{Muon}'s final loss (Table~\ref{tab:reach} in Appendix~\ref{app:setup}), \texttt{LionMuon} needs about $20\%$ fewer FLOPs and $20$ to $30\%$ less time than \texttt{Muon}. \texttt{AdamW} is as fast on PCIe but ends higher, and \texttt{Dion} at rank $1/2$ sends the fewest bytes but ends level with \texttt{Muon}.

%==============================================================================
%==============================================================================

%==============================================================================
\section{Conclusion}
\label{sec:conclusion}
%==============================================================================

We present \texttt{LionMuon}, which takes one \texttt{Muon} step per $P$ iterations and \texttt{Lion} steps in between, with a single dual-EMA momentum buffer and one extra integer $P$. It divides the Newton--Schulz compute and the optimizer all-reduce of \texttt{Muon} by $P$, its sign steps are local, and its state is one buffer per matrix. The complexity bound \eqref{eq: T bound limuon no wd} ties its gain over \texttt{Muon} to the ratios $L_\infty/(\alpha^2 L_2)$ and $\rho_1/(\alpha\rho_\text{nuc})$, and the ratios measured during training put the periods we use ahead. \texttt{LionMuon} with $P{\in}\{2,5\}$ reaches a lower loss than \texttt{Muon}, \texttt{AdamW}, \texttt{Lion} and \texttt{Signum} at 124M and 355M. Under 4-GPU data-parallel training at $1\times$ Chinchilla it ends below every baseline and reaches \texttt{Muon}'s final loss with up to $30\%$ less wall-clock on PCIe. Larger models, multi-node training, an adaptive period and an analysis that accounts for the Newton--Schulz error are the next steps.

\FloatBarrier
\section*{Reproducibility statement}
Algorithm~\ref{alg:lionmuon} is the complete update rule; Appendix~\ref{app:setup} lists every architecture, schedule and hyperparameter, Appendix~\ref{app:heatmap} the tuning grids and selected values. The code with the training scripts, the round-robin distributed protocol and the plotting scripts is at \url{https://github.com/brain-lab-research/lion-muon}. All proofs are in Appendices~\ref{app: missing proofs} and~\ref{app: weight decay}.

\end{mainpart}

\begin{appendixpart}
%==============================================================================

\section{Missing proofs} \label{app: missing proofs}
This appendix collects the missing proofs of Theorem~\ref{thm: main_convergence lionmuon no wd} and Corollary~\ref{col: optimal params limuon no wd}.
We first state and prove two technical lemmas (the descent Lemma \ref{lem: descent lionmuon no wd} and the momentum error bound Lemma \ref{lem: momentum lionmuon no wd}) that are the building blocks of our analysis. Then, we assemble these lemmas into the main proof.

\subsection{Building-block lemmas}
\begin{lemma}[\texttt{LionMuon} Descent Lemma]
\label{lem: descent lionmuon no wd}
Let the objective function $f$ satisfy Assumption \ref{assum:smoothness} with respect to a norm $\|\cdot\|$, and let $\|\cdot\|_\star$ be its dual norm. Then, for update $W_{t+1} =  W_t + \eta_t U_t $ with $ U_t = \text{LMO}_{\|\cdot\|}(\hat{G}_t)$, momentums $M_{t} = \beta_2 M_{t-1} + (1 - \beta_2)G_{t}$ and  $\hat{G}_{t} = \beta_1 M_{t-1} + (1 - \beta_1)G_{t}$, the following bound holds:
\[
\resizebox{\linewidth}{!}{$f(W_{t+1}) \le f(W_t) - \eta_t \cdot \|\nabla f(W_t)\|_\star +  \frac{2\eta_t\beta_{1}}{\beta_2}\| \nabla f(W_t) - M_t \|_\star + 2\eta_t\left|1 - \frac{\beta_{1}}{\beta_2}\right| \|\nabla f(W_t) - G_t \|_\star + \frac{L\eta_t^2}{2}.$}
\]
\end{lemma}
\begin{proof}
We begin our proof with bounding the value $f(W_{t+1})$ after the update using the smoothness Assumption \ref{assum:smoothness}:
\begin{align}
f(W_{t+1}) &= f(W_t + \eta_t U_t) \nonumber \\
&\le f(W_t) +  \eta_t \langle \nabla f(W_t), U_t  \rangle + \frac{L \eta_t^2 }{2} \|U_t\|^2 \nonumber \\
&\le f(W_t) +  \eta_t \langle \nabla f(W_t), U_t \rangle + \frac{L\eta_t^2}{2} \nonumber \\
&= f(W_t) +  \eta_t \langle \hat{G}_t, U_t \rangle +  \eta_t \langle \nabla f(W_t) - \hat{G}_t, U_t  \rangle + \frac{L\eta_t^2}{2}. \nonumber 
\end{align}
Then, we define the optimal matrix $\hat{V}_t := \arg\max_{\|V\| \leq 1 } \langle V , - \nabla f(W_t) \rangle$ and continue bounding:
\begin{align}
f(W_{t+1}) &= f(W_t) +  \eta_t \langle \hat{G}_t, U_t  \rangle +  \eta_t \langle \nabla f(W_t) - \hat{G}_t, U_t \rangle + \frac{L\eta_t^2}{2} \nonumber \\
&\le f(W_t) +  \eta_t \langle \hat{G}_t, \hat{V}_t  \rangle +  \eta_t \langle \nabla f(W_t) - \hat{G}_t, U_t  \rangle + \frac{L\eta_t^2}{2} \nonumber \\
&= f(W_t) +  \eta_t \langle \hat{G}_t, \hat{V}_t - U_t \rangle +  \eta_t \langle \nabla f(W_t), U_t  \rangle + \frac{L\eta_t^2}{2}\nonumber \\
&= f(W_t) +  \eta_t \langle \nabla f(W_t), \hat{V}_t  \rangle +  \eta_t \langle \nabla f(W_t) - \hat{G}_t, U_t - \hat{V}_t \rangle + \frac{L\eta_t^2}{2} \nonumber \\
&\le f(W_t) -  \eta_t  \|\nabla f(W_t)\|_\star +  \eta_t \|\nabla f(W_t) - \hat{G}_t\|_\star \|U_t - \hat{V}_t \| + \frac{L\eta_t^2}{2} \nonumber\\ 
&\le f(W_t) -  \eta_t  \|\nabla f(W_t)\|_\star +   \|\nabla f(W_t) - \hat{G}_t\|_\star 2\eta_t + \frac{L\eta_t^2}{2}. \nonumber
\end{align}

Furthermore, we can switch to the bound with the main momentum $M_t$: 
\begin{align}
\|\nabla f(W_t) - \hat{G}_t\|_\star
&= \|\nabla f(W_t) - M_t + M_t - \hat{G}_t\|_\star \nonumber \\
&= \left\| \nabla f(W_t) - M_t + \left(1 - \frac{\beta_{1}}{\beta_2}\right)(M_t - G_t) \right\|_\star \nonumber \\
&= \left\| \frac{\beta_{1}}{\beta_2}(\nabla f(W_t) - M_t) + \left(1 - \frac{\beta_{1}}{\beta_2}\right)(\nabla f(W_t) - G_t) \right\|_\star \nonumber \\
&\le \frac{\beta_{1}}{\beta_2}\| \nabla f(W_t) - M_t \|_\star + \left|1 - \frac{\beta_{1}}{\beta_2}\right| \|\nabla f(W_t) - G_t \|_\star. \nonumber 
\end{align}
Thus, the final bound is 
\[
\resizebox{\linewidth}{!}{$f(W_{t+1}) \le f(W_t) - \eta_t \cdot  \|\nabla f(W_t)\|_\star +  \frac{2\eta_t\beta_{1}}{\beta_2}\| \nabla f(W_t) - M_t \|_\star + 2\eta_t\left|1 - \frac{\beta_{1}}{\beta_2}\right| \|\nabla f(W_t) - G_t \|_\star + \frac{L\eta_t^2}{2}.$}
\]

\end{proof}

\begin{lemma}[\texttt{LionMuon} Momentum Error Bound]
\label{lem: momentum lionmuon no wd}
Let the objective function $f$ and corrupting noise satisfy Assumptions \ref{assum:smoothness}, \ref{assum:variance}, \ref{assum:norm_eq} with a norm $\|\cdot\|$ and let momentum $M_{\tau}$ be defined as: $M_{\tau} = \beta_2 M_{\tau-1} + (1 - \beta_2)G_{\tau}$. Then, for updates $W_{\tau+1} =  W_\tau + \eta_\tau U_\tau$, the following bound holds
\[
\mathbb{E}[\|E_t\|_\star] \le \beta_2^t\|E_0\|_\star + \frac{LA \beta_2  }{1-\beta_2} + 2\rho_\star \sigma(1-\beta_2)^\frac{\kappa - 1}{\kappa}. 
\]
where $E_t := \nabla f(W_t) - M_t$ and $\max_{\tau\leq t} \{\eta_\tau \cdot    \|U_\tau\|\} \leq A$.
\end{lemma}

\begin{proof}
Using the momentum definition, we write down the recursive step:
\[
\begin{aligned}
E_t &= \nabla f(W_t) - M_t = \nabla f(W_t) - \bigl( \beta_2 M_{t-1} + (1 - \beta_2) G_t \bigr) \\
    &= \beta_2 \nabla f(W_t) + (1 - \beta_2) \nabla f(W_t) - \beta_2 M_{t-1} - (1 - \beta_2) G_t \\
    &= \beta_2 \bigl( \nabla f(W_t) - M_{t-1} \bigr) + (1 - \beta_2) \bigl( \nabla f(W_t) - G_t \bigr) \\
    &= \beta_2 \bigl( \nabla f(W_t) - \nabla f(W_{t-1}) + \nabla f(W_{t-1}) - M_{t-1} \bigr) 
       + (1 - \beta_2) \bigl( \nabla f(W_t) - G_t \bigr) \\
    &= \beta_2 \bigl( \nabla f(W_t) - \nabla f(W_{t-1}) \bigr) 
       + \beta_2 \bigl( \nabla f(W_{t-1}) - M_{t-1} \bigr) 
       + (1 - \beta_2) \bigl( \nabla f(W_t) - G_t \bigr).
\end{aligned}
\]
Further, we use the notations $S_t = \nabla f(W_{t}) - G_t$ and $R_t = \nabla f(W_t) - \nabla f(W_{t-1})$ to unroll the recursion:
\[
E_t = \beta_2 E_{t-1} + (1-\beta_2)S_t + \beta_2 R_t
= \beta_2^t E_0 + \sum_{j=0}^{t-1} \beta_2^j\bigl[(1-\beta_2)S_{t-j} + \beta_2 R_{t-j}\bigr].
\]
Now, we observe that
\[
\|R_{t-j}\|_\star = \|\nabla f(W_{t-j}) - \nabla f(W_{t-j-1})\|_\star \le L\|W_{t-j} - W_{t-j-1}\| = L\eta_{t-j-1}  \|U_{t-j-1}\| \le LA .
\]

Therefore, we estimate using the norm equivalence Assumption~\ref{assum:norm_eq}:
\begin{align*}
\mathbb{E}[\|E_t\|_\star]
&\le \beta_2^t\cdot \|E_0\|_\star + \mathbb{E}\Biggl[\Biggl\|\sum_{j=0}^{t-1}\beta_2^j\bigl[(1-\beta_2)S_{t-j}+\beta_2 R_{t-j}\bigr]\Biggr\|_\star\Biggr] \\
&\le \beta_2^t \cdot \|E_0\|_\star + \sum_{j=0}^{t-1}\beta_2^{j+1}\mathbb{E}[\|R_{t-j}\|_\star]  + \mathbb{E}\Biggl[\Biggl\|\sum_{j=0}^{t-1}\beta_2^j(1-\beta_2)S_{t-j}\Biggr\|_\star\Biggr] \\
&\le \beta_2^t\cdot \|E_0\|_\star + \frac{LA\beta_2}{1-\beta_2}  + \rho_\star \left(\mathbb{E}\Biggl[\Biggl\|\sum_{j=0}^{t-1}\beta_2^j(1-\beta_2)S_{t-j}\Biggr\|_F^\kappa\Biggr]\right)^\frac{1}{\kappa}. \nonumber
\end{align*}
For the linear combination of corrupting noises, we apply a batching lemma on the reduction of the $\kappa$-th moment, proposed and developed in works \citep{kornilov2023accelerated, hubler2024gradient}:
\begin{lemma}\label{lem: batching}
    Let $X_1, \dots, X_B$ be a matrix martingale difference sequence  (i.e. $\EE[X_j|X_{j-1}, \dots, X_1] = 0$ for $1 < j \leq B$) such that $\EE[\|X_j\|_F^{\kappa}|X_{j-1}, \dots, X_1] \leq \sigma_j^{\kappa}$ for $1 < \kappa \leq 2$.
Then, we have
\[\EE\left[ \left\|  \sum\limits_{j=1}^B X_j\right\|_F^{\kappa}\right] \leq 2 \sum_{j=1}^B \sigma_j^\kappa .\]
The constant is $1$ for $\kappa = 2$, by orthogonality of martingale differences, and $2$ for $1 < \kappa < 2$, by the von Bahr--Esseen inequality.
\end{lemma}
Namely, we treat the sequence  $\{\beta_2^j(1-\beta_2) \cdot S_{t-j}\}_{j=0}^{t-1} $ as the required martingale difference sequence  with $\sigma_j = \beta_2^j(1-\beta_2)\sigma$ and apply Lemma \ref{lem: batching}:
\begin{align*}
\mathbb{E}\Biggl[\Biggl\|\sum_{j=0}^{t-1}\beta_2^j(1-\beta_2)S_{t-j}\Biggr\|_F^\kappa\Biggr]
&\leq 2\sum_{j=0}^{t-1}\beta_2^{\kappa j}(1-\beta_2)^\kappa\sigma^\kappa \\
&\le 2\sigma^\kappa(1-\beta_2)^\kappa\sum_{j=0}^{t-1}\beta_2^{\kappa j} \\
&\le \frac{2\sigma^\kappa(1-\beta_2)^\kappa}{1-\beta_2^{\kappa}}. \nonumber
\end{align*}
Hence, we get
\[
\mathbb{E}[\|E_t\|_\star] \le \beta_2^t\cdot \|E_0\|_\star + \frac{LA\beta_2 }{1-\beta_2} + 2^{\frac1\kappa}\rho_\star \sigma{\frac{1-\beta_2}{(1-\beta_2^\kappa)^\frac1\kappa}}. \nonumber
\]
Since $0 < 1 - \beta_2 \le 1 - \beta_2^\kappa$ and $2^{1/\kappa} \le 2$, we further simplify the bound to:
\[
\mathbb{E}[\|E_t\|_\star] \le \beta_2^t\cdot \|E_0\|_\star + \frac{LA \beta_2 }{1-\beta_2} + 2\rho_\star\sigma(1-\beta_2)^\frac{\kappa - 1}{\kappa}.
\]

\end{proof}

\subsection{Proof of \texttt{LionMuon} Convergence Theorem  \ref{thm: main_convergence lionmuon no wd} }

\begin{proof}
We divide the iteration indices $t \in \{0, \dots, T-1\}$ into two disjoint sets: the set of \texttt{Muon} steps $S_{\text{muon}} = \{t \mid t \equiv 0 \pmod P\}$ and the set of block \texttt{Lion} steps $S_{\text{lion}} = \{t \mid t \not\equiv 0 \pmod P\}$. \\

\textbf{Step 1: Analysis of the \texttt{Muon} Steps ($t \in S_{\text{muon}}$).} For $t \in S_{\text{muon}}$, the update utilizes the spectral norm $\|\cdot\|_2$.  To use Lemmas \ref{lem: descent lionmuon no wd} and \ref{lem: momentum lionmuon no wd}, we find the uniform upper bound constant $A_2$ such that $\max_{\tau\leq t}\{\eta_\tau \|U_\tau\|_2\} \leq A_2$ for all previous steps $\tau \le t$: 

\begin{itemize}
    \item If $\tau \in S_{\text{muon}}$, then all updates $ U_\tau = \text{LMO}_{ \|\cdot\|_2}(\hat{G}_\tau)$ are bounded by $\|U_\tau\|_2 = 1$, and the stepsize is $\eta_\tau = \eta_{M}$. 

    \item If $\tau \in S_{\text{lion}}$, then the updates  $ U_\tau = \text{LMO}_{ \|\cdot\|_\infty}(\hat{G}_\tau)$ utilize the infinity norm LMO, yielding $\|U_\tau\|_{\infty} = 1$. Using the norm equivalence ($\|W\|_2 \le \sqrt{mn}\|W\|_{\infty}$), we have $\|U_\tau\|_2 \le \sqrt{mn}$ and stepsize  $\eta_\tau = \eta_{L}$.

\end{itemize} 
Taking the maximum over these two cases for $P \in (1, \infty)$, we get $A_2 = \max\{\eta_M, \sqrt{mn} \cdot \eta_L\} $. When $P = 1$, all $\tau$ steps belong only to $S_\text{muon}$ and $A_2 = \eta_M.$

Thus, combining Lemmas \ref{lem: descent lionmuon no wd} and \ref{lem: momentum lionmuon no wd} with $A_2$ and dual variance factor $\rho_{\text{nuc}}$, we bound the gradient dual norm:
\begin{align}
\eta_M \cdot \EE[\|\nabla f(W_t)\|_{\text{nuc}}]  &\le \EE[f(W_t)]  - \EE[f(W_{t+1})] +  \frac{2\eta_M \beta_{1}}{\beta_2} \EE[\| \nabla f(W_t) - M_t \|_{\text{nuc}}] \nonumber \\
&+ 2\eta_M\left|1 - \frac{\beta_{1}}{\beta_2}\right| \EE[\|\nabla f(W_t) - G_t \|_{\text{nuc}}] + \frac{L_2\eta^2_M}{2} \nonumber \\
& \leq  \EE[f(W_t)]  - \EE[f(W_{t+1})] \notag \\
&+  \frac{2\eta_M \beta_{1}}{\beta_2} \left( \beta_2^t \|E_0\|_{\text{nuc}} + \frac{L_2A_2  \beta_2}{1-\beta_2} + 2\rho_{\text{nuc}}\sigma(1-\beta_2)^\frac{\kappa - 1}{\kappa}\right) \nonumber \\
&+ 2\eta_M \left|1 - \frac{\beta_{1}}{\beta_2}\right| \rho_{\text{nuc}} \sigma + \frac{L_2\eta^2_M}{2} =: \EE[f(W_t)]  - \EE[f(W_{t+1})] + \text{Error}_t^\text{muon}. \label{eq: muon step bound no wd}
\end{align}

\textbf{Step 2: Analysis of the \texttt{Lion} Steps ($t \in S_{\text{lion}}$).} For $t \in S_{\text{lion}}$, the update utilizes the infinite norm $\|\cdot\|_\infty$. To use Lemmas \ref{lem: descent lionmuon no wd} and \ref{lem: momentum lionmuon no wd}, we find the uniform upper bound constant $A_\infty$ such that $\max_{\tau\leq t}\{ \eta_\tau \|U_\tau\|_\infty\} \leq A_\infty$ for all past steps $\tau \le t$: 

\begin{itemize}
    \item If $\tau \in S_{\text{muon}}$, then all updates $ U_\tau = \text{LMO}_{ \|\cdot\|_2}(\hat{G}_\tau)$ are bounded by $\|U_\tau\|_\infty \leq \|U_\tau\|_2 =  1$ and the stepsize is $\eta_\tau = \eta_{M}$. 

    \item If $\tau \in S_{\text{lion}}$, then the updates  $ U_\tau = \text{LMO}_{ \|\cdot\|_\infty}(\hat{G}_\tau)$ utilize the infinity norm LMO, yielding $\|U_\tau\|_{\infty} = 1$ and stepsize  $\eta_\tau = \eta_{L}$. 

\end{itemize} 
Taking the maximum over these two cases for $P \in (1, \infty)$, we get $A_\infty  = \max(\eta_M, \eta_L)$. When $P = \infty$, all $\tau$ steps belong only to $S_\text{lion}$ and $A_\infty = \eta_L.$

Similarly combining Lemmas \ref{lem: descent lionmuon no wd} and \ref{lem: momentum lionmuon no wd} with $A_\infty$ and dual variance factor $\rho_{\text{1}}$, we bound the gradient dual norm:
\begin{align}
\eta_L \cdot \EE[\|\nabla f(W_t)\|_{1}]
& \leq  \EE[f(W_t)]  - \EE[f(W_{t+1})] \notag \\
&+  \frac{2\eta_L\beta_{1}}{\beta_2} \left( \beta_2^t \|E_0\|_1 + \frac{L_\infty A_\infty\beta_2}{1-\beta_2} + 2\rho_{1}\sigma(1-\beta_2)^\frac{\kappa - 1}{\kappa}\right) \nonumber \\
&+ 2\eta_L\left|1 - \frac{\beta_{1}}{\beta_2}\right| \rho_{1} \sigma + \frac{L_\infty \eta_L^2}{2} =: \EE[f(W_t)]  - \EE[f(W_{t+1})] + \text{Error}_t^\text{lion}. \label{eq: lion step bound no wd}
\end{align}

\textbf{Step 3: Telescoping Sum.} We sum the bounds for \texttt{Muon} \eqref{eq: muon step bound no wd} and \texttt{Lion} \eqref{eq: lion step bound no wd} steps over $t=0$ to $T-1$. Note that we group the terms over $\frac{T}{P}$ periods of length $P$, and the total numbers of each step type are $|S_{\text{muon}}| = \frac{T}{P}$ and $|S_{\text{lion}}| = \frac{T(P-1)}{P}$:
\begin{align*}
 &\sum_{i=0}^{\frac{T}{P} - 1} \left(\eta_M \cdot \EE[\|\nabla f(W_{i\cdot P})\|_{\text{nuc}}]  + \sum_{j = 1}^{P-1}[\eta_{L} \cdot \EE[\|\nabla f(W_{i\cdot P + j})\|_{1}] ] \right) \\
 &\qquad \leq f(W_0) - f_\star + \sum_{t \in S_{\text{muon}}}\text{Error}_t^\text{muon} + \sum_{t \in S_{\text{lion}}} \text{Error}_t^\text{lion}.
\end{align*}
For the left-hand side, we consider the minimal period-averaged gradient dual norm:
\begin{align*}
&\sum_{i=0}^{\frac{T}{P} - 1} \left(\eta_M \cdot \EE[\|\nabla f(W_{i\cdot P})\|_{\text{nuc}}]  + \sum_{j = 1}^{P-1}[\eta_{L} \cdot \EE[\|\nabla f(W_{i\cdot P + j})\|_{1}] ] \right) \\
&= \sum_{i=0}^{\frac{T}{P} - 1} (\eta_M + \eta_L(P-1)) \cdot \underset{:= \EE[\|\overline{\nabla} f(W_{i\cdot P})\|]}{\underbrace{\frac{\left(\eta_M \cdot \EE[\|\nabla f(W_{i\cdot P})\|_{\text{nuc}}]  + \sum_{j = 1}^{P-1}[\eta_{L} \cdot \EE[\|\nabla f(W_{i\cdot P + j})\|_{1}] ] \right)}{\eta_M + \eta_L(P-1)} }}\\
&\geq \sum_{i=0}^{\frac{T}{P} - 1} (\eta_M + \eta_L(P-1)) \cdot \min_i \{\EE[\|\overline{\nabla} f(W_{i\cdot P})\|]\}\\
&= \frac{T}{P} (\eta_M + \eta_L(P-1)) \cdot \min_i \{\EE[\|\overline{\nabla} f(W_{i\cdot P})\|]\} = T \cdot \bar{\eta}  \cdot \min_i \{\EE[\|\overline{\nabla} f(W_{i\cdot P})\|]\},
\end{align*}
where the period-averaged stepsize is $\bar{\eta} := \frac{\eta_{M}}{P} + \frac{\eta_{L}(P-1)}{P}.$

Note that when $P=1$ or $P = \infty$ the minimal averaged norm becomes the minimal dual norm over all intermediate points:
\begin{gather*}
\min_i \{\EE[\|\overline{\nabla} f(W_{i\cdot P})\|]\} = \min_i \{\EE[\|\nabla f(W_{i})\|_\text{nuc}]\} \quad \text{for } P = 1, \\
\min_i \{\EE[\|\overline{\nabla} f(W_{i\cdot P})\|]\} = \min_i \{\EE[\|\nabla f(W_{i})\|_1]\} \quad \text{for } P = \infty.
\end{gather*} 

For the right-hand side, we apply the geometric series upper bound $\sum_{t=0}^{T-1} \beta_2^t \le \frac{1}{1-\beta_2}$ for the intermediate momentum errors. Grouping the constant terms matching the lengths of sets $S_{\text{muon}}$ and $S_{\text{lion}}$  and dividing the entire inequality by $T\bar{\eta}$, we obtain the overall bound: 
\begin{align}
   \min_i \{\EE[\|\overline{\nabla} f(W_{i\cdot P})\|]\}  &\leq \frac{\Delta_0}{ \bar{\eta}T } +  \frac{2\eta_M \beta_{1}}{\beta_2} \frac{\|E_0\|_{\text{nuc}}}{T \bar{\eta}(1 - \beta_2)} + \frac1P  \frac{2\eta_M \beta_{1}}{\beta_2} \frac{L_2A_2  \beta_2}{(1-\beta_2)  \bar{\eta}} \nonumber \\
   &+  \frac1P  \frac{4\eta_M \beta_{1} }{\beta_2  \bar{\eta}}\rho_{\text{nuc}}\sigma(1-\beta_2)^\frac{\kappa - 1}{\kappa}
    +  \frac1P  \frac{2\eta_M}{ \bar{\eta}} \left|1 - \frac{\beta_{1}}{\beta_2}\right| \rho_{\text{nuc}} \sigma + \frac1P  \frac{L_2\eta^2_M}{2 \bar{\eta}} \nonumber \\
     &+  \frac{2\eta_L \beta_{1}}{\beta_2} \frac{\|E_0\|_{1}}{T \bar{\eta}(1 - \beta_2)} + \frac{2\eta_L \beta_{1}}{\beta_2} \frac{P-1}{P}\frac{L_\infty  A_\infty \beta_2}{(1-\beta_2)  \bar{\eta}}\nonumber \\
     &+  \frac{P-1}{P} \frac{4\eta_L \beta_{1} }{\beta_2  \bar{\eta}}\rho_{1}\sigma(1-\beta_2)^\frac{\kappa - 1}{\kappa} +  2 \frac{P-1}{P} \frac{\eta_L}{ \bar{\eta}} \left|1 - \frac{\beta_{1}}{\beta_2}\right| \rho_{1} \sigma \nonumber \\
     &+ \frac{P-1}{P}\frac{L_\infty \eta^2_L}{2 \bar{\eta}}.  \notag
\end{align}
We can combine the momentum $\beta_2$ terms:
$$ \frac1P  \frac{L_2\eta^2_M}{2  \bar{\eta}} \leq  \frac1P  \frac{L_2 A_2 \eta_M }{2(1-\beta_2)  \bar{\eta}} \leq  \frac1P  \frac{L_2 A_2 \eta_M}{2(1-\beta_2)  \bar{\eta}^2} \bar{\eta}$$
and
$$\frac{P-1}{P}\frac{L_\infty \eta^2_L}{2 \bar{\eta}} \leq  \frac{P-1}{P}\frac{  L_\infty \eta_L  A_\infty}{2 (1-\beta_2)  \bar{\eta}}   \leq \frac{P-1}{P}\frac{  L_\infty \eta_L  A_{\infty}}{2 (1-\beta_2)  \bar{\eta}^2} \bar{\eta}.$$
For the initial norm terms, we go back to the sums before the geometric series bound: the \texttt{Muon} and \texttt{Lion} steps share one series, and $\|\cdot\|_{\text{nuc}} \le \|\cdot\|_1$, so
$$ \frac{2 \beta_{1}}{\beta_2 T \bar{\eta}} \Bigl( \sum_{t \in S_{\text{muon}}} \beta_2^t \eta_M \|E_0\|_{\text{nuc}} + \sum_{t \in S_{\text{lion}}} \beta_2^t \eta_L \|E_0\|_{1} \Bigr) \leq \frac{2 \beta_{1}}{\beta_2} \frac{\max\{\eta_L, \eta_M\}\|E_0\|_{1}}{T \bar{\eta}(1 - \beta_2)}.$$
When $P =1$ or $P = \infty$, only one of the terms appears, and the bound still holds true.

Next, we  define the period-averaged noise and smoothness constants:
\begin{align}
    \bar{\rho} &:= \frac{\eta_M}{P \bar{\eta}} \rho_{\text{nuc}} + \frac{(P-1)\eta_L}{P \bar{\eta}} \rho_{1} , \notag \\
    \bar{L} &: = \frac{\eta_M A_2}{P \bar{\eta}^2} L_2 + \frac{(P-1)\eta_L A_\infty}{P \bar{\eta}^2} L_\infty. \label{eq: Inter L proofs}
\end{align}
Employing the averaged constants, we further simplify the bound:
\begin{align}
   \min_i \{\EE[\|\overline{\nabla} f(W_{i\cdot P})\|]\} &\leq \frac{\Delta_0}{ \bar{\eta}T } +  \frac{2 \beta_{1}}{\beta_2} \frac{\eta_{\max} \|E_0\|_{1}}{\bar{\eta}T (1 - \beta_2)} +    \frac{4 \bar{L} \bar{\eta} }{(1-\beta_2)} \nonumber \\
   &+  \frac{4\beta_{1} }{\beta_2}\bar{\rho}  \sigma (1-\beta_2)^\frac{\kappa - 1}{\kappa} + 2 \left|1 - \frac{\beta_{1}}{\beta_2}\right| \bar{\rho}  \sigma .  \nonumber
\end{align}

\end{proof}

\subsection{Proof of Optimal Parameters Corollary \ref{col: optimal params limuon no wd}} \label{sec: cor proof no wd}
\begin{proof}
    
In Theorem \ref{thm: main_convergence lionmuon no wd}, we obtained the convergence bound of \texttt{LionMuon} algorithm under arbitrary parameters:
\begin{align}
   \min_i \{\EE[\|\overline{\nabla} f(W_{i\cdot P})\|]\} &\leq \frac{\Delta_0}{ \bar{\eta}T } +  \frac{2 \beta_{1}}{\beta_2} \frac{\eta_{\max} \|E_0\|_{1}}{\bar{\eta}T (1 - \beta_2)}+    \frac{4 \bar{L} \bar{\eta} }{(1-\beta_2)} \nonumber \\
   &+  \frac{4\beta_{1} }{\beta_2}\bar{\rho}  \sigma (1-\beta_2)^\frac{\kappa - 1}{\kappa} + 2 \left|1 - \frac{\beta_{1}}{\beta_2}\right| \bar{\rho}  \sigma ,  \label{eq: lionmuon bound no params} \\
   \EE[\|\overline{\nabla} f(W_{i\cdot P})\|] &:= \frac{\left(\eta_M \cdot \EE[\|\nabla f(W_{i\cdot P})\|_{\text{nuc}}]  + \sum_{j = 1}^{P-1}[\eta_{L} \cdot \EE[\|\nabla f(W_{i\cdot P + j})\|_{1}] ] \right)}{\eta_M + (P-1)\eta_L} \notag \\
   &\geq \min_{j}\EE[\|\nabla f(W_{i\cdot P + j})\|_\text{nuc}]. \nonumber
\end{align}

\textbf{Fixed period $P \in (1,\infty)$.}  To achieve accuracy $\varepsilon$, we choose the optimal horizon $T$, momentums $\beta_1, \beta_2$, stepsizes $\eta_L$ and $\eta_M = \alpha \eta_L$, whereas  period $P$ and stepsizes scale $\alpha$ are treated as hyperparameters. 

First, we pick the smaller momentum $\beta_1 \leq \beta_2$  close to $\beta_2$ to limit the last term in \eqref{eq: lionmuon bound no params}:
$$2 \left|1 - \frac{\beta_{1}}{\beta_2}\right| \bar{\rho}\sigma \leq \frac{\varepsilon}{8} \quad \Longrightarrow \quad \beta_1 = \beta_2 \cdot [\max\{1 - \frac{\varepsilon}{16 \bar{\rho}\sigma}, 0\}, 1].$$
Now, all ratios $\frac{\beta_1}{\beta_2}$ can be upper-bounded by $1$. We continue with the noise term:
$$\frac{4\beta_{1} }{\beta_2}\bar{\rho}\sigma (1-\beta_2)^\frac{\kappa - 1}{\kappa} \leq 4\bar{\rho}\sigma (1-\beta_2)^\frac{\kappa - 1}{\kappa} \leq \frac{\varepsilon}{8} \quad \Longrightarrow 1 - \beta_2 = \left(\frac{\varepsilon}{32 \bar{\rho}\sigma}\right)^\frac{\kappa}{\kappa - 1}.$$
We also keep $\beta_2 \geq \tfrac12$, so that $\beta_2 > 0$ and the smoothness term below stays simple, i.e., $1 - \beta_2 = \min\{\left(\frac{\varepsilon}{32 \bar{\rho}\sigma}\right)^\frac{\kappa}{\kappa - 1}, \tfrac12 \}.$ Next, we upper-bound the third term:
$$\frac{4 \bar{L}  \bar{\eta} }{(1-\beta_2)}  = \frac{4 \bar{L} (\frac{\alpha}{P} + \frac{P-1}{P}) \eta_L  }{(1-\beta_2)} \leq \frac{\varepsilon}{8} \quad \Longrightarrow \quad \eta_{L} = \frac{\varepsilon (1 - \beta_2)}{32 (\frac{\alpha}{P} + \frac{P-1}{P}) \bar{L}}, \eta_M = \alpha \eta_L. $$
To proceed to the second term, we note that the period-averaged stepsize $\bar{\eta} := \frac{\eta_{M}}{P} + \frac{\eta_{L}(P-1)}{P}$ can be lower-bounded by $\bar{\eta} \geq \frac1P \max\{\eta_M, \eta_L\} = \frac{\eta_L}{P} \max\{1, \alpha\}$ as a convex combination: 
$$ \frac{2 \beta_{1}}{\beta_2} \frac{\eta_{\max} \|E_0\|_{1}}{\bar{\eta}T (1 - \beta_2)} \leq \frac{2P \beta_{1}}{\beta_2} \frac{\eta_{\max} \|E_0\|_{1}}{\eta_{\max} T (1 - \beta_2)} \leq  \frac{2P\|E_0\|_{1}}{T (1 - \beta_2)} \leq \frac{\varepsilon}{8} \quad \Longrightarrow \quad $$
$$T \geq \frac{16 P \|E_0\|_{1}}{(1 - \beta_2)\varepsilon} = 16 P \|E_0\|_{1} \cdot \max \left\{\frac{(32 \bar{\rho}\sigma)^\frac{\kappa}{\kappa - 1}}{\varepsilon^\frac{2\kappa - 1}{\kappa - 1}}, \frac{2}{\varepsilon} \right\}.$$
Finally, we bound the first term:
$$\frac{\Delta_0}{ \bar{\eta}T } \leq \frac{\varepsilon}{8} \quad \Longrightarrow\quad T \geq \frac{8\Delta_0 }{\varepsilon \bar{\eta}  } = 2^8 \cdot \bar{L}\Delta_0 \cdot \max \left\{\frac{(32 \bar{\rho}\sigma)^\frac{\kappa}{\kappa - 1} }{\varepsilon^\frac{3\kappa - 2}{\kappa - 1}}, \frac{2}{\varepsilon^2} \right\}.$$
The bound of the previous term is dominated by this one when $\varepsilon \leq 16\bar{L}\Delta_0/(P\|E_0\|_1)$, which the corollary assumes. Hence, we keep only the last bound:
$$T = O\left( \bar{L}\Delta_0 \cdot \max \left\{\frac{( \bar{\rho}\sigma)^\frac{\kappa}{\kappa - 1} }{\varepsilon^\frac{3\kappa - 2}{\kappa - 1}}, \frac{1 }{\varepsilon^2} \right\} \right).$$

\textbf{Cases $P=1$ and $P = \infty$.} In these cases, the proof is identical with constants $\bar{L} = L_2, \bar{\rho} = \rho_{\text{nuc}}, A_2 = \eta_M$ or $\bar{L} = L_\infty, \bar{\rho} = \rho_{1}, A_\infty = \eta_L$ until the stepsize pick. The considered stepsizes become $ \bar{\eta} = \eta_M$ or $ \bar{\eta} = \eta_L$.

\end{proof}

\subsection{Remark about the constants for dense matrices} \label{app: remark about ref smoothness}
In our proofs, we use the worst-case norm inequalities \eqref{eq: norm relations} which cause extra conservative factors in the obtained bound from Theorem \ref{thm: main_convergence lionmuon no wd}. Fortunately, the gradients and update matrices during LLM training tend to have a dense structure, as we also observe in our experiments (Table~\ref{tab:constants}). Thus, this case is worth a separate analysis.  

We call an update  matrix $U_t = \text{LMO}_{ \|\cdot\|}(\hat{G}_\tau)$ \textit{dense}, if we have an approximate equivalence:
\begin{align}
    \|U_t\|_2 \approx \alpha_U \|U_t\|_\infty \quad \text{for some constant $\alpha_U \lesssim \sqrt{mn}$,} \label{eq: a for dense}
\end{align} 
Here $\alpha_U$ is the ratio of the update matrices, not the gradient ratio $\alpha$ of \eqref{eq: dense constants eq}. Now, we can estimate the refined constants $A_2$ and $A_\infty$ in the proof of Theorem \ref{thm: main_convergence lionmuon no wd} at \textbf{Steps} $1$ \textbf{and} $2$.

\textbf{Step 1: Refined analysis of the \texttt{Muon} Steps ($t \in S_{\text{muon}}$).} For $t \in S_{\text{muon}}$, the update utilizes the spectral norm $\|\cdot\|_2$.  To use Lemmas \ref{lem: descent lionmuon no wd} and \ref{lem: momentum lionmuon no wd}, we estimate the uniform upper bound constant $A_2$ such that $\max_{\tau\leq t}\{\eta_\tau \|U_\tau\|_2\} \leq A_2$ for all previous steps $\tau \le t$: 

\begin{itemize}
    \item If $\tau \in S_{\text{muon}}$, then all updates $ U_\tau = \text{LMO}_{ \|\cdot\|_2}(\hat{G}_\tau)$ are bounded by $\|U_\tau\|_2 = 1$, and the stepsize is $\eta_\tau = \eta_{M}$. 

    \item If $\tau \in S_{\text{lion}}$, then the updates  $ U_\tau = \text{LMO}_{ \|\cdot\|_\infty}(\hat{G}_\tau)$ utilize the infinity norm LMO, yielding $\|U_\tau\|_{\infty} = 1$. Using the norm equality \eqref{eq: a for dense}, we have $\|U_\tau\|_2 \approx \alpha_U$ and stepsize  $\eta_\tau = \eta_{L}$.

\end{itemize} 
Taking the maximum over these two cases for $P \in (1, \infty)$, we get $A_2 = \max\{\eta_M, \alpha_U \cdot \eta_L\} $. When $P = 1$, all $\tau$ steps belong only to $S_\text{muon}$ and $A_2 = \eta_M.$

\textbf{Step 2: Refined analysis of the \texttt{Lion} Steps ($t \in S_{\text{lion}}$).}  For $t \in S_{\text{lion}}$, the update utilizes the infinite norm $\|\cdot\|_\infty$. To use Lemmas \ref{lem: descent lionmuon no wd} and \ref{lem: momentum lionmuon no wd}, we estimate the uniform upper bound constant $A_\infty$ such that $\max_{\tau\leq t}\{ \eta_\tau \|U_\tau\|_\infty\} \leq A_\infty$ for all past steps $\tau \le t$: 

\begin{itemize}
    \item If $\tau \in S_{\text{muon}}$, then all updates $ U_\tau = \text{LMO}_{ \|\cdot\|_2}(\hat{G}_\tau)$ are bounded by $\|U_\tau\|_\infty \approx \frac1{\alpha_U}\|U_\tau\|_2 =  \frac1{\alpha_U}$ and the stepsize is $\eta_\tau = \eta_{M}$. 

    \item If $\tau \in S_{\text{lion}}$, then the updates  $ U_\tau = \text{LMO}_{ \|\cdot\|_\infty}(\hat{G}_\tau)$ utilize the infinity norm LMO, yielding $\|U_\tau\|_{\infty} = 1$ and stepsize  $\eta_\tau = \eta_{L}$. 

\end{itemize} 
Taking the maximum over these two cases for $P \in (1, \infty)$, we get $A_\infty  = \max(\frac1{\alpha_U} \eta_M, \eta_L)$. When $P = \infty$, all $\tau$ steps belong only to $S_\text{lion}$ and $A_\infty = \eta_L.$

\textbf{Refined interpolated smoothness.} With new refined uniform constants $A_2$ and $A_\infty$, we can similarly derive new interpolated smoothness from \eqref{eq: Inter L proofs}: 
\begin{align*}
\bar{L} &= \frac{\eta_M A_2}{P \bar{\eta}^2} L_2 + \frac{(P-1)\eta_L A_\infty}{P \bar{\eta}^2} L_\infty \\
&= \frac{\eta_M \max\{\eta_M, \alpha_U \eta_L\}}{P \bar{\eta}^2} L_2 + \frac{(P-1)\eta_L \max(\frac1{\alpha_U} \eta_M, \eta_L)}{P \bar{\eta}^2} L_\infty.
\end{align*}
Next, we apply the scale $\eta_M/\eta_L = \alpha_U$ to equalize the different norms and get new smoothness:
\begin{eqnarray}
    \bar{L} &=&  \frac{\eta_M \max\{\eta_M, \alpha_U \cdot \eta_L\}}{P \bar{\eta}^2} L_2 + \frac{(P-1)\eta_L \max(\frac1{\alpha_U} \eta_M, \eta_L)}{P \bar{\eta}^2} L_\infty \notag \\
    &=&  \frac{\eta_M \max\{\eta_M, \frac{\alpha_U}{\alpha_U}\eta_M\}}{P \bar{\eta}^2} L_2 + \frac{(P-1)\eta_L \max(\frac{\alpha_U}{\alpha_U}\eta_L, \eta_L)}{P \bar{\eta}^2} L_\infty \notag \\
    &\approx&  \frac{\eta_M^2 }{P \bar{\eta}^2} L_2 + \frac{(P-1)\eta_L^2}{P \bar{\eta}^2} L_\infty. \label{eq: refined smoothness formula}
\end{eqnarray}
The refined smoothness \eqref{eq: refined smoothness formula} more naturally and smoothly interpolates between pure \texttt{Muon} $L_2$ and pure \texttt{Lion} $L_\infty$ when going from $P=1$ to $P=\infty$.

\section{Weight Decay Analysis}
\label{app: weight decay}

\paragraph{Notations.} We denote the closed $\|\cdot\|$-norm ball of radius $r$ by $B_{\|\cdot\|}(r) := \{S \in \mathbb{R}^{m \times n} : \|S\| \le r\}$ and rewrite LMO as $\mathrm{LMO}_{B_{\|\cdot\|}(r)}(G) = \arg\min_{S \in B_{\|\cdot\|}(r)} \langle G, S \rangle$. We use the spectral norm LMO to calculate the matrix-sign operation $\mathrm{LMO}_{B_{\|\cdot\|_2}(r)}(G) = -r\cdot \msign(G)$ and the infinity norm LMO to calculate the element-wise sign $\mathrm{LMO}_{B_{\|\cdot\|_\infty}(r)}(G) = -r\cdot \mathrm{sign}(G)$.

\paragraph{Constrained optimization view.}
With these new notations, LMO update \eqref{eq: LMO step no wd} with a weight decay $\lambda > 0$ can be restated as a Frank-Wolfe step:
$$\resizebox{\linewidth}{!}{$W_{t+1} \;=\; W_t + \eta_t\,\mathrm{LMO}_{B_{\|\cdot\|}(1)}(\hat{G}_t)- \eta_t \lambda W_t
\quad \Leftrightarrow \quad  W_{t+1} \;=\; (1 - \eta_t \lambda)W_t + \eta_t\lambda\,\mathrm{LMO}_{B_{\|\cdot\|}(1/\lambda)}(\hat{G}_t).$}$$
This Frank-Wolfe algorithm solves the \textit{constrained} optimization problem \citep{chen2023lion, chen2025muon, sfyraki2025lions}: $$\min_{W \in B_{\|\cdot\|}(\frac1\lambda)} f(W).$$ 

As a convergence criterion, we use the Frank-Wolfe gap for a set $\mathcal{C} \subseteq \mathbb{R}^{m \times n}$:
\[
\mathcal{G}_{\mathcal{C}}(W) := \max{_{V \in \mathcal{C}}} \langle V - W, -\nabla f(W) \rangle
\]
which equals exactly zero at the KKT points of $\mathcal{C}$.

\textit{Our \texttt{LionMuon} iterates between working within the $B_{\|\cdot\|_2}(1/\lambda)$ ball at \texttt{Muon} iterations and within the larger $B_{\|\cdot\|_\infty}(1/\lambda)$ ball at \texttt{Lion} ones. Hence, our method preserves fast convergence to \texttt{Muon} inner KKT points, while also being able to reach \texttt{Lion} KKT points away from the \texttt{Muon} ball. }

We generalize Theorem \ref{thm: main_convergence lionmuon no wd} to bound the minimal smaller Frank-Wolfe gap $\mathcal{G}_{B_{\|\cdot\|_2}(1/\lambda)}(W)$ on a smaller set and obtain almost identical convergence bound (the differences are {\color{purple} highlighted}):
\begin{align}
   \min_{t} \{\lambda \cdot \EE[\mathcal{G}_{B_{\|\cdot\|_2}(\frac1\lambda)}(W_t)]\} &\leq \frac{\Delta_0}{ \bar{\eta}T }  +    \frac{{{\color{purple}8 \bar{L}}} \bar{\eta} }{\color{purple} \min\{ (1-\beta_2), 1/\sqrt{mn}\}} +  \frac{4\beta_{1} }{\beta_2}\bar{\rho}  \sigma (1-\beta_2)^\frac{\kappa - 1}{\kappa} \notag \\
   &+ 2 \left|1 - \frac{\beta_{1}}{\beta_2}\right| \bar{\rho}  \sigma +  \frac{2 \beta_{1}}{\beta_2} \frac{\eta_{\max} \|E_0\|_{1}}{\bar{\eta}T (1 - \beta_2)},  \notag
\end{align}
where the \textit{new smoothness} ${\color{purple} \bar{L}}
:= \tfrac{\eta_M {\color{purple}\sqrt{mn} \cdot \eta_{\max}}}{P \bar{\eta}^2}  L_2 + \tfrac{(P-1)\eta_L \eta_{\max}}{P \bar{\eta}^2} L_\infty$ is applied. The full Theorem \ref{thm:lionmuon} with Corollary \ref{col: optimal params limuon wd} about the optimal parameters are located below.

We extend the prior work \citep{sfyraki2025lions} which provides analysis of simple momentums-equipped LMO updates with weight decay and heavy-tailed noise. We consider a wider class of switching LMO updates and obtain better noise dependence for pure \texttt{Muon} and \texttt{Lion} in Corollary \ref{col: optimal params limuon wd}.  

\paragraph{Key differences from the non-weight-decay case:} 
\begin{itemize}
    \item \textbf{New gap metric.} First, we cannot apply all norm inequalities to the gaps in different sets. We can only guarantee that gap on a smaller $\|\cdot\|_2$-ball is lower than gap on a $\|\cdot\|_\infty$-ball.  
    Thus, we do not bound the weighted gap in the bounds, but we still use learning rates scale to equalize the smoothness and noise constants.   
    
    Second, due to switching between the balls, some matrices $W_t$ can go out of the \texttt{Muon} ball, and the gap can become negative. Nevertheless, it is still an informative metric as both zero and negative gaps indicate that no direction towards the \texttt{Muon} ball will yield improvement. 

    \item \textbf{New interpolation.} When matrix $W_t$ goes out of the \texttt{Muon} ball during \texttt{Lion} steps, it may slow the convergence for next \texttt{Muon} steps. For this reason, a bit worse factors ${\color{purple}\sqrt{mn}}$ appear in the new weight decay bound and smoothness. 
\end{itemize}

\textit{All other discussions about optimal parameters, learning rates scale and period remain the same. }

\subsection{Building-block lemmas}

First, we prove the modified version of building-blocks lemmas.

\begin{lemma}[\texttt{LionMuon} Descent Lemma with Weight Decay]
\label{lem:descent}
Let the objective function $f$ satisfy Assumption~\ref{assum:smoothness} with respect to a norm $\|\cdot\|$, and let $\|\cdot\|_\star$ be its dual norm.
Then, for the update $W_{t+1} = (1 - \lambda\eta_t) W_t + \lambda\eta_t U_t$ with $U_t = \mathrm{LMO}_{B_{\|\cdot\|}(1/\lambda)}(\hat{G}_t)$, momentums $M_t = \beta_2 M_{t-1} + (1-\beta_2) G_t$  and $\hat{G}_t = \beta_1 M_{t-1} + (1-\beta_1) G_t$, the following bound holds:
\begin{align*}
f(W_{t+1}) \le\;& f(W_t) - \lambda\eta_t \cdot \mathcal{G}_{B_{\|\cdot\|}(1/\lambda)}(W_t)
+ \frac{2\eta_t\beta_1}{\beta_2}\| \nabla f(W_t) - M_t \|_\star \\
&+ 2\eta_t\left|1 - \frac{\beta_1}{\beta_2}\right| \|\nabla f(W_t) - G_t \|_\star
+ 2 L C^2 \eta_t^2,
\end{align*}
where the Frank-Wolfe gap is $\mathcal{G}_{B_{\|\cdot\|}(1/\lambda)}(W_t) := \max_{V \in B_{\|\cdot\|}(1/\lambda)} \langle V - W_t, -\nabla f(W_t) \rangle$, and update and intermediate matrices are confined to $\max\{\|W_t\|, \|U_t\|\} \leq C/\lambda$.
\end{lemma}

\begin{proof}
We begin by bounding $f(W_{t+1})$ using the smoothness Assumption \ref{assum:smoothness}:
\begin{align*}
f(W_{t+1}) &= f(W_t + \lambda \eta_t (U_t - W_t)) \\
&\le f(W_t) + \lambda \eta_t \langle \nabla f(W_t), U_t - W_t \rangle + \frac{L \lambda^2 \eta_t^2}{2} \|U_t - W_t\|^2 \\
&\le f(W_t) + \lambda \eta_t \langle \nabla f(W_t), U_t - W_t \rangle + \frac{L}{2} \lambda^2 C^2 \eta_t^2 \cdot \frac{4}{\lambda^2} \\
&= f(W_t) + \lambda \eta_t \langle \hat{G}_t, U_t - W_t \rangle + \lambda \eta_t \langle \nabla f(W_t) - \hat{G}_t, U_t - W_t \rangle + 2L C^2 \eta_t^2.
\end{align*}
We define $\hat{V}_t := \arg\max_{V \in B_{\|\cdot\|}(1/\lambda)} \langle V - W_t, -\nabla f(W_t) \rangle$ and continue:
\begin{align*}
f(W_{t+1}) &= f(W_t) + \lambda \eta_t \langle \hat{G}_t, U_t - W_t \rangle + \lambda \eta_t \langle \nabla f(W_t) - \hat{G}_t, U_t - W_t \rangle + 2L C^2 \eta_t^2 \\
&\le f(W_t) + \lambda \eta_t \langle \hat{G}_t, \hat{V}_t - W_t \rangle + \lambda \eta_t \langle \nabla f(W_t) - \hat{G}_t, U_t - W_t \rangle + 2L C^2 \eta_t^2 \\
&= f(W_t) + \lambda \eta_t \langle \hat{G}_t, \hat{V}_t - U_t \rangle + \lambda \eta_t \langle \nabla f(W_t), U_t - W_t \rangle + 2L C^2 \eta_t^2 \\
&= f(W_t) + \lambda \eta_t \langle \nabla f(W_t), \hat{V}_t - W_t \rangle + \lambda \eta_t \langle \nabla f(W_t) - \hat{G}_t, U_t - \hat{V}_t \rangle + 2L C^2 \eta_t^2 \\
&\le f(W_t) - \lambda \eta_t \mathcal{G}_{B_{\|\cdot\|}(1/\lambda)}(W_t) + \lambda \eta_t \|\nabla f(W_t) - \hat{G}_t\|_\star \|U_t - \hat{V}_t\| + 2L C^2 \eta_t^2 \\
&\le f(W_t) - \lambda \eta_t \mathcal{G}_{B_{\|\cdot\|}(1/\lambda)}(W_t) + \lambda \|\nabla f(W_t) - \hat{G}_t\|_\star \cdot \frac{2\eta_t}{\lambda} + 2L C^2 \eta_t^2.
\end{align*}
We can switch to a bound using the main momentum $M_t$:
\begin{align*}
\|\nabla f(W_t) - \hat{G}_t\|_\star
&= \|\nabla f(W_t) - M_t + M_t - \hat{G}_t\|_\star \\
&= \left\| \nabla f(W_t) - M_t + \left(1 - \frac{\beta_1}{\beta_2}\right)(M_t - G_t) \right\|_\star \\
&= \left\| \frac{\beta_1}{\beta_2}(\nabla f(W_t) - M_t) + \left(1 - \frac{\beta_1}{\beta_2}\right)(\nabla f(W_t) - G_t) \right\|_\star \\
&\le \frac{\beta_1}{\beta_2}\| \nabla f(W_t) - M_t \|_\star + \left|1 - \frac{\beta_1}{\beta_2}\right| \|\nabla f(W_t) - G_t \|_\star.
\end{align*}
Finally, we yield the required bound:
\begin{align*}
f(W_{t+1}) \le\;& f(W_t) - \lambda\eta_t \cdot \mathcal{G}_{B_{\|\cdot\|}(1/\lambda)}(W_t)
+ \frac{2\eta_t\beta_1}{\beta_2}\| \nabla f(W_t) - M_t \|_\star \\
&+ 2\eta_t\left|1 - \frac{\beta_1}{\beta_2}\right| \|\nabla f(W_t) - G_t \|_\star
+ 2 L C^2 \eta_t^2.
\end{align*}
\end{proof}

\begin{lemma}[\texttt{LionMuon} Momentum Error Bound with Weight Decay]
\label{lem:momentum}
Let the objective function $f$ and corrupting noise  satisfy Assumptions~\ref{assum:smoothness}, \ref{assum:variance}, \ref{assum:norm_eq} with norm $\|\cdot\|$ and let momentum $M_\tau$ be defined as $M_\tau = \beta_2 M_{\tau-1} + (1-\beta_2) G_\tau$.
Then, for updates $W_{\tau+1} = (1 - \lambda\eta_\tau) W_\tau + \lambda\eta_\tau U_\tau$ with $U_\tau = \mathrm{LMO}_{B_{\|\cdot\|}(1/\lambda)}(\hat{G}_\tau)$, the following bound holds:
\[
\EE[\|E_t\|_\star] \le \beta_2^t \|E_0\|_\star + \frac{2 L A \beta_2}{1-\beta_2} + 2\rho_\star \sigma (1-\beta_2)^{\frac{\kappa-1}{\kappa}},
\]
where $E_t := \nabla f(W_t) - M_t$ and $\max_{\tau \le t}\{\eta_\tau\|W_\tau\|, \eta_\tau\|U_\tau\|\} \le A/\lambda$.
\end{lemma}

\begin{proof}
Using the momentum definition, we write down the recursive step:
\[
\begin{aligned}
E_t &= \nabla f(W_t) - M_t = \nabla f(W_t) - \bigl( \beta_2 M_{t-1} + (1 - \beta_2) G_t \bigr) \\
&= \beta_2 \nabla f(W_t) + (1 - \beta_2) \nabla f(W_t) - \beta_2 M_{t-1} - (1 - \beta_2) G_t \\
&= \beta_2 \bigl( \nabla f(W_t) - M_{t-1} \bigr) + (1 - \beta_2) \bigl( \nabla f(W_t) - G_t \bigr) \\
&= \beta_2 \bigl( \nabla f(W_t) - \nabla f(W_{t-1}) + \nabla f(W_{t-1}) - M_{t-1} \bigr)
   + (1 - \beta_2) \bigl( \nabla f(W_t) - G_t \bigr) \\
&= \beta_2 \bigl( \nabla f(W_t) - \nabla f(W_{t-1}) \bigr)
   + \beta_2 \bigl( \nabla f(W_{t-1}) - M_{t-1} \bigr)
   + (1 - \beta_2) \bigl( \nabla f(W_t) - G_t \bigr).
\end{aligned}
\]
Using the notations $S_t = \nabla f(W_t) - G_t$ and $R_t = \nabla f(W_t) - \nabla f(W_{t-1})$, we unroll the recursion:
\[
E_t = \beta_2 E_{t-1} + (1-\beta_2) S_t + \beta_2 R_t
= \beta_2^t E_0 + \sum_{j=0}^{t-1} \beta_2^j \bigl[(1-\beta_2) S_{t-j} + \beta_2 R_{t-j}\bigr].
\]
By smoothness, we have:
\begin{align*}
\|R_{t-j}\|_\star
&= \|\nabla f(W_{t-j}) - \nabla f(W_{t-j-1})\|_\star \\
&\le L \|W_{t-j} - W_{t-j-1}\|
= L \lambda \eta_{t-j-1} \|U_{t-j-1} - W_{t-j-1}\|
\le 2 L A.
\end{align*}
We continue with the norm-equivalence Assumption \ref{assum:norm_eq} and Jensen's inequality for math expectation:
\begin{align*}
\EE[\|E_t\|_\star]
&\le \beta_2^t \cdot \|E_0\|_\star + \EE\left[ \left\| \sum_{j=0}^{t-1} \beta_2^j \bigl[(1-\beta_2) S_{t-j} + \beta_2 R_{t-j}\bigr] \right\|_\star \right] \\
&\le \beta_2^t \cdot \|E_0\|_\star + \sum_{j=0}^{t-1} \beta_2^{j+1} \EE[\|R_{t-j}\|_\star] + \EE\left[ \left\| \sum_{j=0}^{t-1} \beta_2^j (1-\beta_2) S_{t-j} \right\|_\star \right] \\
&\le \beta_2^t \cdot \|E_0\|_\star + \frac{2 L A \beta_2}{1-\beta_2} + \rho_\star \left( \EE\left[ \left\| \sum_{j=0}^{t-1} \beta_2^j (1-\beta_2) S_{t-j} \right\|_F^{\kappa} \right] \right)^{1/\kappa}.
\end{align*}
We treat the sequence $\{\beta_2^j(1-\beta_2) \cdot S_{t-j}\}_{j=0}^{t-1} $ as a martingale difference sequence  with $\sigma_j = \beta_2^j(1-\beta_2)\sigma$  and apply the batching lemma \ref{lem: batching} :
\begin{align*}
\EE\left[ \left\| \sum_{j=0}^{t-1} \beta_2^j (1-\beta_2) S_{t-j} \right\|_F^{\kappa} \right]
&\le 2\sum_{j=0}^{t-1} \beta_2^{\kappa j} (1-\beta_2)^{\kappa} \sigma^{\kappa}  \\
&\le 2\sigma^{\kappa} (1-\beta_2)^{\kappa} \sum_{j=0}^{t-1} \beta_2^{\kappa j} \\
&\le \frac{2\sigma^{\kappa} (1-\beta_2)^{\kappa}}{1 - \beta_2^{\kappa}}.
\end{align*}
Hence, we have the required bound:
\[
\EE[\|E_t\|_\star] \le \beta_2^t \cdot \|E_0\|_\star + \frac{2 L A \beta_2}{1-\beta_2} + 2^{1/\kappa}\rho_\star \sigma \cdot \frac{1-\beta_2}{(1-\beta_2^{\kappa})^{1/\kappa}}.
\]
Since $0 < 1 - \beta_2 \le 1 - \beta_2^{\kappa}$ and $2^{1/\kappa} \le 2$, we further simplify:
\[
\EE[\|E_t\|_\star] \le \beta_2^t \cdot \|E_0\|_\star + \frac{2 L A \beta_2}{1-\beta_2} + 2\rho_\star \sigma (1-\beta_2)^{\frac{\kappa-1}{\kappa}}.
\]
\end{proof}

\subsection{\texttt{LionMuon} Convergence Theorem with weight decay}

\begin{theorem}[Convergence of \texttt{LionMuon}, $\lambda > 0$]
\label{thm:lionmuon}
Let the objective function $f$ satisfy Assumption \ref{assum:smoothness} with respect to $\|\cdot\|_2$ with constant $L_2$ and with respect to $\|\cdot\|_{\infty}$ with constant $L_{\infty}$. Let noise Assumptions \ref{assum:variance} and \ref{assum:norm_eq} hold with noise constants $\sigma$, $\rho_\text{nuc}$ and $\rho_{1}$. Fix a horizon $T$, period $P \in [1, \infty]$, weight decay $\lambda > 0$, momentum parameters $\beta_1 \in [0, 1)$, $\beta_2 \in (0, 1)$ and learning rates $\eta_{M}$ and $\eta_{L}$ with $\lambda \max\{\eta_M, \eta_L\} \le 1$, and let $W_0 \in B_{\|\cdot\|_\infty}(1/\lambda)$ ($W_0 \in B_{\|\cdot\|_2}(1/\lambda)$ when $P = 1$). 

Define the period-averaged learning rate, noise level and smoothness:
\begin{eqnarray}
\bar{\eta} := \tfrac{\eta_{M}}{P} + \tfrac{(P-1) \eta_{L}}{P},
\quad
\bar{\rho} := \tfrac{\eta_M}{P \bar{\eta}} \rho_{\text{nuc}} + \tfrac{(P-1)\eta_L}{P \bar{\eta}} \rho_{1},
\quad
\bar{L} := \tfrac{\eta_M {\eta}_{\max} C_2 }{P \bar{\eta}^2} L_2 + \tfrac{(P-1)\eta_L \eta_{\max}}{P \bar{\eta}^2} L_\infty,  \label{eq: period avr constants wd}
\end{eqnarray}
where $\eta_{\max} = \max\{\eta_M, \eta_L\}$ and $C_2=  \sqrt{mn}$  for intermediate $P \in (1, \infty)$ with the boundary cases $\eta_{\max} = \eta_M, C_2 = 1$ at $P=1$ and $ \eta_{\max} = \eta_L, C_2 = 1$ at $P=\infty$.

Then, our \texttt{LionMuon} algorithm starting with $\Delta_0 := f(W_0) - f_\star, E_0 = \nabla f(W_0) - M_0$ guarantees the following bound on the period-averaged Frank-Wolfe gap norm:
\begin{align}
   \min_t \EE\bigl[\lambda \cdot \mathcal{G}_{B_{\|\cdot\|_2}(1/\lambda)}(W_t)\bigr] &\leq \frac{\Delta_0}{ \bar{\eta}T }  +    \frac{8 \bar{L} \bar{\eta} }{{\min\{1-\beta_2, 1/C_2\}}} +  \frac{4\beta_{1} }{\beta_2}\bar{\rho}  \sigma (1-\beta_2)^\frac{\kappa - 1}{\kappa} \notag \\
   &+ 2 \left|1 - \frac{\beta_{1}}{\beta_2}\right| \bar{\rho}  \sigma +  \frac{2 \beta_{1}}{\beta_2} \frac{\eta_{\max} \|E_0\|_{1}}{\bar{\eta}T (1 - \beta_2)},  \notag 
\end{align}
For $P = \infty$ (pure \texttt{Lion}), the same bound holds with the larger Frank-Wolfe gap $\mathcal{G}_{B_{\|\cdot\|_\infty}(1/\lambda)}(W_t)$.
\end{theorem}

\begin{proof}
We divide the iteration indices $t \in \{0, \ldots, T-1\}$ into two disjoint sets: the set of \texttt{Muon} steps $S_\text{muon} = \{t \mid t \equiv 0 \pmod P\}$ and the set of \texttt{Lion} steps $S_\text{lion} = \{t \mid t \not\equiv 0 \pmod P\}$.

\textbf{Step 1: Analysis of the \texttt{Muon} steps ($t \in S_\text{muon}$).}
For $t \in S_\text{muon}$, the update uses the spectral norm $\|\cdot\|_2$.
To apply Lemmas~\ref{lem:descent} and~\ref{lem:momentum}, we need the  bound $C_2$ such that $\max\{\|W_t\|_2, \|U_t\|_2\} \le C_2/\lambda$, and the uniform bound $A_2$ such that $\max_{\tau\leq t}\{\eta_\tau \|W_\tau\|_2, \eta_\tau \|U_\tau\|_2\} \le A_2/\lambda$ for all past steps $\tau \le t$:
\begin{itemize}
    \item If $\tau \in S_\text{muon}$, all updates $U_\tau = \mathrm{LMO}_{B_{\|\cdot\|_2}(1/\lambda)}(\hat{G}_\tau)$ are bounded by $\|U_\tau\|_2 = 1/\lambda$ with stepsize $\eta_\tau = \eta_M$.
    \item If $\tau \in S_\text{lion}$, the updates $U_\tau = \mathrm{LMO}_{B_{\|\cdot\|_\infty}(1/\lambda)}(\hat{G}_\tau)$ use the infinity-norm LMO, yielding $\|U_\tau\|_\infty = 1/\lambda$. By norm equivalence, we have $\|U_\tau\|_2 \le \sqrt{mn}/\lambda$ with stepsize $\eta_\tau = \eta_L$.
    \item When $P \in (1, \infty)$, all $W_\tau$ lie in the \texttt{Lion} ball $B_{\|\cdot\|_\infty}(1/\lambda)$, since $W_0$ does and every step is a convex combination of $W_\tau$ with a point of this ball ($\lambda \eta_\tau \le 1$), so we have $\|W_\tau\|_2 \le \sqrt{mn}\,\|W_\tau\|_\infty \le \sqrt{mn}/\lambda$ with alternating stepsizes $\eta_\tau \le \max\{\eta_M, \eta_L\}$.
    \item When $P = 1$, all $W_\tau$ lie in the \texttt{Muon} ball $B_{\|\cdot\|_2}(1/\lambda)$ for the same reason, so we have $\|W_\tau\|_2 \le 1/\lambda$ with single stepsize $\eta_\tau = \eta_M$.
\end{itemize}
Taking the maximum over these cases, we have  $\eta_{\max} = \max\{\eta_M, \eta_L\}, C_2 = \sqrt{mn}, A_2 = C_2 \cdot \eta_{\max}$ if $P \in (1, \infty)$ and $\eta_{\max} = \eta_M, C_2 = 1, A_2 = \eta_M$ if $P = 1$ (no \texttt{Lion} steps).

Combining Lemmas~\ref{lem:descent} and~\ref{lem:momentum} with $C_2, A_2$ and dual variance factor $\rho_{\mathrm{nuc}}$, we bound the gap:
\begin{align}
\lambda \eta_M \cdot \EE[\mathcal{G}_{B_{\|\cdot\|_2}(1/\lambda)}(W_t)]
&\le \EE[f(W_t)] - \EE[f(W_{t+1})] + \frac{2\eta_M \beta_1}{\beta_2} \EE[\|\nabla f(W_t) - M_t\|_{\mathrm{nuc}}] \nonumber \\
& + 2\eta_M \left|1 - \frac{\beta_1}{\beta_2}\right| \EE[\|\nabla f(W_t) - G_t\|_{\mathrm{nuc}}] + 2 L_2 C_2^2 \eta_M^2 \nonumber \\
&\le \EE[f(W_t)] - \EE[f(W_{t+1})] \nonumber \\
& + \frac{2\eta_M \beta_1}{\beta_2}\left( \beta_2^t \|E_0\|_{\mathrm{nuc}} + \frac{2 L_2 C_2 \eta_{\max} \beta_2}{1-\beta_2} + 2\rho_{\mathrm{nuc}} \sigma (1-\beta_2)^{\frac{\kappa-1}{\kappa}} \right) \nonumber \\
& + 2\eta_M \left|1 - \frac{\beta_1}{\beta_2}\right| \rho_{\mathrm{nuc}} \sigma + 2 L_2 C_2^2 \eta_M^2 \nonumber \\
&=: \EE[f(W_t)] - \EE[f(W_{t+1})] + \mathrm{Error}_t^{\text{muon}}. \label{eq: muon step bound wd}
\end{align}

\textbf{Step 2: Analysis of the \texttt{Lion} steps ($t \in S_\text{lion}$).}
For $t \in S_\text{lion}$, the update uses the infinity norm $\|\cdot\|_\infty$.
To apply Lemmas~\ref{lem:descent} and~\ref{lem:momentum}, we need the  bound $C_\infty$ such that $\max\{\|W_t\|_\infty, \|U_t\|_\infty\} \le C_\infty/\lambda$, and the uniform bound $A_\infty$ such that $\max_{\tau\leq t}\{\eta_\tau \|W_\tau\|_\infty, \eta_\tau \|U_\tau\|_\infty\} \le A_\infty/\lambda$ for all past steps $\tau \le t$:
\begin{itemize}
    \item If $\tau \in S_\text{muon}$, all updates $U_\tau = \mathrm{LMO}_{B_{\|\cdot\|_2}(1/\lambda)}(\hat{G}_\tau)$ satisfy $\|U_\tau\|_\infty \le \|U_\tau\|_2 = 1/\lambda$ with stepsize $\eta_\tau = \eta_M$.
    \item If $\tau \in S_\text{lion}$, the updates $U_\tau = \mathrm{LMO}_{B_{\|\cdot\|_\infty}(1/\lambda)}(\hat{G}_\tau)$ use the infinity-norm LMO, yielding $\|U_\tau\|_\infty = 1/\lambda$ with stepsize $\eta_\tau = \eta_L$.
    \item All steps $W_\tau$ lie in the ball $B_{\|\cdot\|_\infty}(1/\lambda)$, so we have $\|W_\tau\|_\infty \le 1/\lambda$ with alternating stepsizes $\eta_\tau \le \max\{\eta_M, \eta_L\}$ if $P \in (1, \infty)$ or single stepsize $\eta_\tau = \eta_L$ if $P = \infty$.
\end{itemize}
Taking the maximum, we get $\eta_{\max} = \max(\eta_M, \eta_L), C_\infty = 1, A_\infty = \eta_{\max} $ if $P \in (1, \infty)$ and $\eta_{\max} = \eta_L, C_\infty = 1, A_\infty = \eta_L$ if $P = \infty$ (no \texttt{Muon} steps).

Similarly combining Lemmas~\ref{lem:descent} and~\ref{lem:momentum} with $C_\infty, A_\infty$ and dual variance factor $\rho_1$, we bound the Frank-Wolfe gap:
\begin{align*}
\lambda \eta_L \cdot \EE[\mathcal{G}_{B_{\|\cdot\|_\infty}(1/\lambda)}(W_t)]
&\le \EE[f(W_t)] - \EE[f(W_{t+1})]\notag \\
&+ \frac{2\eta_L \beta_1}{\beta_2}\left( \beta_2^t \|E_0\|_1 + \frac{2 L_\infty \eta_{\max} \beta_2}{1-\beta_2} + 2\rho_1 \sigma (1-\beta_2)^{\frac{\kappa-1}{\kappa}} \right) \\
&+ 2\eta_L \left|1 - \frac{\beta_1}{\beta_2}\right| \rho_1 \sigma + 2 L_\infty \eta_L^2.
\end{align*}
Since $\|X\|_\infty \le \|X\|_2$ \eqref{eq: norm relations}, we have an inclusion  $B_{\|\cdot\|_2}(1/\lambda) \subseteq B_{\|\cdot\|_\infty}(1/\lambda)$.
Hence, the maximum defining the Frank-Wolfe gap $\mathcal{G}_{B_{\|\cdot\|_\infty}(1/\lambda)}(W_t)$ over the $\|\cdot\|_\infty$-ball is taken over a larger set and we can safely lower-bound $\mathcal{G}_{B_{\|\cdot\|_2}(1/\lambda)}(W_t) \le \mathcal{G}_{B_{\|\cdot\|_\infty}(1/\lambda)}(W_t)$. The final bound is
\begin{align}
\lambda \eta_L \cdot \EE[\mathcal{G}_{B_{\|\cdot\|_2}(1/\lambda)}(W_t)]
&\le \EE[f(W_t)] - \EE[f(W_{t+1})] \nonumber \\
& + \frac{2\eta_L \beta_1}{\beta_2}\left( \beta_2^t \|E_0\|_1 + \frac{2 L_\infty \eta_{\max} \beta_2}{1-\beta_2} + 2\rho_1 \sigma (1-\beta_2)^{\frac{\kappa-1}{\kappa}} \right) \nonumber \\
& + 2\eta_L \left|1 - \frac{\beta_1}{\beta_2}\right| \rho_1 \sigma + 2 L_\infty \eta_L^2 \nonumber \\
&=: \EE[f(W_t)] - \EE[f(W_{t+1})] + \mathrm{Error}_t^{\text{lion}}. \label{eq: lion step bound wd}
\end{align}

\textbf{Step 3: Telescoping sum.}
We sum the bounds for \texttt{Muon}~\eqref{eq: muon step bound wd} and \texttt{Lion}~\eqref{eq: lion step bound wd} steps over $t = 0, \ldots, T-1$. The number of steps of each type are $|S_\text{muon}| = T/P$ and $|S_\text{lion}| = T(P-1)/P$:
\begin{align}
\sum_{t=0}^{T-1} \bigl( \mathbf{1}_{t \in S_\text{muon}} \eta_M + \mathbf{1}_{t \in S_\text{lion}} \eta_L \bigr) \cdot \EE[\lambda \cdot \mathcal{G}_{B_{\|\cdot\|_2}(1/\lambda)}(W_t)] &\le f(W_0) - f_\star \notag \\
&+ \sum_{t \in S_\text{muon}} \mathrm{Error}_t^{\text{muon}} + \sum_{t \in S_\text{lion}} \mathrm{Error}_t^{\text{lion}}. \notag
\end{align}
The coefficient on the left-hand side sums exactly to $T(\eta_M/P + \eta_L (P-1)/P) = T\bar{\eta}$, so we lower-bound it by $T\bar{\eta} \cdot \min_t \EE[\lambda \cdot \mathcal{G}_{B_{\|\cdot\|_2}(1/\lambda)}(W_t)]$.
On the right-hand side, we apply the geometric series bound $\sum_{t=0}^{T-1} \beta_2^t \le 1/(1-\beta_2)$ for the intermediate momentum errors. Grouping constants matched to $|S_\text{muon}|$ and $|S_\text{lion}|$ and dividing by $T\bar{\eta}$, we get:
\begin{align}
\min_t \EE[\lambda \cdot \mathcal{G}_{B_{\|\cdot\|_2}(1/\lambda)}(W_t)]
&\le \frac{\Delta_0}{\bar{\eta}T} + \frac{2\eta_M \beta_1}{\beta_2} \frac{\|E_0\|_{\mathrm{nuc}}}{T \bar{\eta}(1-\beta_2)} \notag \\
&+ \frac{1}{P} \frac{2\eta_M \beta_1}{\beta_2} \frac{2 L_2 C_2 \eta_{\max} \beta_2}{(1-\beta_2)\bar{\eta}} \nonumber \\
&+ \frac{1}{P} \frac{4\eta_M \beta_1}{\beta_2 \bar{\eta}} \rho_{\mathrm{nuc}} \sigma (1-\beta_2)^{\frac{\kappa-1}{\kappa}}  + \frac{1}{P} \frac{2\eta_M}{\bar{\eta}} \left|1 - \frac{\beta_1}{\beta_2}\right| \rho_{\mathrm{nuc}} \sigma \nonumber \\
&+ \frac{1}{P} \frac{2 L_2 C_2^2 \eta_M^2}{ \bar{\eta}} \nonumber \\
& + \frac{2\eta_L \beta_1}{\beta_2} \frac{\|E_0\|_1}{T \bar{\eta}(1-\beta_2)} + \frac{2\eta_L \beta_1}{\beta_2} \frac{P-1}{P} \frac{2 L_\infty \eta_{\max} \beta_2}{(1-\beta_2)\bar{\eta}} \nonumber \\
&+ \frac{P-1}{P} \frac{4\eta_L \beta_1}{\beta_2 \bar{\eta}} \rho_1 \sigma (1-\beta_2)^{\frac{\kappa-1}{\kappa}}  + 2 \frac{P-1}{P} \frac{\eta_L}{\bar{\eta}} \left|1 - \frac{\beta_1}{\beta_2}\right| \rho_1 \sigma \notag \\
&+ \frac{P-1}{P} \frac{2L_\infty \eta_L^2}{\bar{\eta}}. \nonumber
\end{align}
Next, we unite the momentum $\beta_2$ terms:
\[
\frac{1}{P} \frac{2L_2 C_2^2 \eta_M^2}{ \bar{\eta}} \le \frac{1}{P} \frac{2 L_2 C_2 \eta_{\max} \eta_M}{\min\{1-\beta_2, 1/C_2\} \bar{\eta}} = \frac{1}{P} \frac{2 \bar{\eta} \cdot L_2 C_2 \eta_{\max} \eta_M}{\min\{1-\beta_2, 1/C_2\} \bar{\eta}^2} \]
and
\[ \quad \frac{P-1}{P} \frac{2 L_\infty \eta_L^2}{\bar{\eta}} \le \frac{P-1}{P} \frac{2 L_\infty \eta_L \eta_{\max}}{\min\{1-\beta_2, 1/C_2\} \bar{\eta} } = \frac{P-1}{P} \frac{2 \bar{\eta} \cdot L_\infty \eta_L \eta_{\max}}{\min\{1-\beta_2, 1/C_2\} \bar{\eta}^2 }.
\]
For the initial norm terms, we go back to the sums before the geometric series bound: the \texttt{Muon} and \texttt{Lion} steps share one series, and $\|\cdot\|_{\mathrm{nuc}} \le \|\cdot\|_1$, so
\[
\frac{2 \beta_1}{\beta_2 T \bar{\eta}} \Bigl( \sum_{t \in S_\text{muon}} \beta_2^t \eta_M \|E_0\|_{\mathrm{nuc}} + \sum_{t \in S_\text{lion}} \beta_2^t \eta_L \|E_0\|_1 \Bigr) \le \frac{2 \eta_{\max} \beta_1}{\beta_2} \frac{\|E_0\|_1}{T \bar{\eta}(1-\beta_2)}.
\]
When $P = 1$ or $P = \infty$, only one of the two terms appears, and the bound still holds.

Next, we define the period-averaged noise level and smoothness:
\begin{align*}
\bar{\rho} &:= \frac{\eta_M}{P \bar{\eta}} \rho_{\mathrm{nuc}} + \frac{(P-1)\eta_L}{P \bar{\eta}} \rho_1, \\
\bar{L} &:= \frac{\eta_M\eta_{\max} C_2}{P \bar{\eta}^2} L_2 + \frac{(P-1)\eta_L\eta_{\max}}{P \bar{\eta}^2} L_\infty.
\end{align*}
Using these constants, we further simplify:
\begin{align}
\min_t \EE[\lambda \cdot \mathcal{G}_{B_{\|\cdot\|_2}(1/\lambda)}(W_t)]
&\le \frac{\Delta_0}{\bar{\eta}T} + \frac{2\beta_1}{\beta_2} \frac{\eta_{\max} \|E_0\|_1}{\bar{\eta}T(1-\beta_2)} + \frac{8\bar{L}\bar{\eta}}{\min\{1-\beta_2, 1/C_2\}} \nonumber \\
& + \frac{4\beta_1}{\beta_2} \bar{\rho}\sigma (1-\beta_2)^{\frac{\kappa-1}{\kappa}} + 2\left|1 - \frac{\beta_1}{\beta_2}\right| \bar{\rho}\sigma. \nonumber
\end{align}
For the pure-\texttt{Lion} case $P = \infty$, the same bound holds for the larger Frank-Wolfe gap $\min_t \EE[\lambda \cdot \mathcal{G}_{B_{\|\cdot\|_\infty}(1/\lambda)}(W_t)]$.
\end{proof}

\subsection{\texttt{LionMuon} Optimal Parameters Corollary with weight decay}
\begin{corollary}[Optimal Parameters for \texttt{LionMuon}, $\lambda > 0$]
\label{col: optimal params limuon wd}
Let the objective function $f$ and the noise satisfy Assumptions \ref{assum:smoothness}, \ref{assum:variance} and \ref{assum:norm_eq}  with the period-averaged constants $\bar{L}$, $\sigma$ and $\bar{\rho}$  defined in \eqref{eq: period avr constants wd}.

\begin{itemize}[leftmargin=15pt]
    \item  Fix a weight decay $\lambda >0$, period $P \in (1, \infty)$ and learning rates scale $\alpha = \eta_M/\eta_L$.

To achieve accuracy $\min_t \EE[\lambda \cdot \mathcal{G}_{B_{\|\cdot\|_2}(\frac1\lambda)}(W_t)] \leq \varepsilon$ with $\varepsilon \leq 16\bar{L}\Delta_0/(P\|E_0\|_1)$, from $W_0$ as in Theorem~\ref{thm:lionmuon} and with $\lambda \max\{\eta_M, \eta_L\} \le 1$ for the learning rates below, our \texttt{LionMuon} requires $T$ iterations
\begin{equation}
    T = O\left(\bar{L}\Delta_0 \cdot \max \left\{\frac{( \bar{\rho}\sigma)^\frac{\kappa}{\kappa - 1} }{\varepsilon^\frac{3\kappa - 2}{\kappa - 1}}, \frac{C_2}{\varepsilon^2} \right\}\right),  \label{eq: T bound limuon wd}
\end{equation}
with the optimal parameters:
\begin{gather*}
1 - \beta_2 = \min\{\left(\frac{\varepsilon}{32 \bar{\rho}\sigma}\right)^\frac{\kappa}{\kappa - 1}, \frac{1}{2C_2} \}, \quad \beta_1 \in \beta_2\cdot[ \max\{1 - \frac{\varepsilon}{16 \bar{\rho}\sigma}, 0\},1] \\
\eta_{L} = \frac{\varepsilon (1 - \beta_2)}{64 (\frac{\alpha}{P} + \frac{P-1}{P}) \cdot \bar{L}}.
\end{gather*}
\item Pure \texttt{Muon} ($P=1$) and \texttt{Lion}  ($P=\infty$) keep the same momentums $\beta_1, \beta_2$, number of iterations $T$ and only single learning rate $\eta_M = \frac{\varepsilon (1 - \beta_2)}{64  \cdot L_2}$ or $\eta_L = \frac{\varepsilon (1 - \beta_2)}{64  \cdot L_\infty}$ .

\item We can set single-EMA $\beta_1 = \beta_2$ to get optimal parameters for our \texttt{SignMuon} with weight decay.

\end{itemize}

\end{corollary}
\begin{proof} The proof copies the proof of Corollary \ref{col: optimal params limuon no wd} in Appendix \ref{sec: cor proof no wd}, with the interpolated smoothness \eqref{eq: period avr constants wd} in place of \eqref{eq: period avr constants} and the extra condition $1 - \beta_2 \leq \frac{1}{2C_2}$ on the momentum.
    
\end{proof}

\section{Estimated constants during training}
\label{app:constants}

We record the quantities of Assumptions~\ref{assum:norm_eq} and~\ref{assum:smoothness} every $100$ steps over $10{,}000$ steps of a 124M run with $P{=}2$ on each dataset ($100$ records per run), as the median over the 2D parameters, and report the mean over the records in Table~\ref{tab:constants}: the gradient-norm ratio $\alpha = \|G_t\|_1/\|G_t\|_\text{nuc}$ \eqref{eq: dense constants eq}, the noise levels $\rho_\text{nuc} \approx \|G_t - M_t\|_\text{nuc}/\|G_t - M_t\|_\text{F}$ and $\rho_1 \approx \|G_t - M_t\|_1/\|G_t - M_t\|_\text{F}$ (Assumption~\ref{assum:norm_eq}, with momentum as a less noisy gradient estimate), and the smoothness constants $L_2 \approx \|G_{t+1} - G_t\|_\text{nuc}/\|W_{t+1} - W_t\|_2$ and $L_\infty \approx \|G_{t+1} - G_t\|_1/\|W_{t+1} - W_t\|_\infty$ (Assumption~\ref{assum:smoothness}).

\begin{table}[h]
\centering
\caption{Constants of the bound, measured during 124M training with $P{=}2$. At each step we take the median over the 2D parameters, and the table gives the mean over the run (100 records per run). $L_2$ and $L_\infty$ are single-batch estimates and therefore noisier. The last two columns are the ratios the bound depends on.}
\label{tab:constants}
\small
\begin{tabular}{@{}lccccccc@{}}
\toprule
 & $\alpha$ & $\rho_\text{nuc}$ & $\rho_1$ & $L_2$ & $L_\infty$ & $L_\infty/(\alpha^2 L_2)$ & $\rho_1/(\alpha\rho_\text{nuc})$ \\
\midrule
FineWeb & $67$ & $13.6$ & $894$ & $0.79$ & $857$ & $0.24$ & $0.98$ \\
WikiText-103 & $75$ & $10.9$ & $835$ & $1.46$ & $1312$ & $0.16$ & $1.02$ \\
\bottomrule
\end{tabular}
\end{table}

\begin{figure}[h]
\centering
\includegraphics[width=\textwidth]{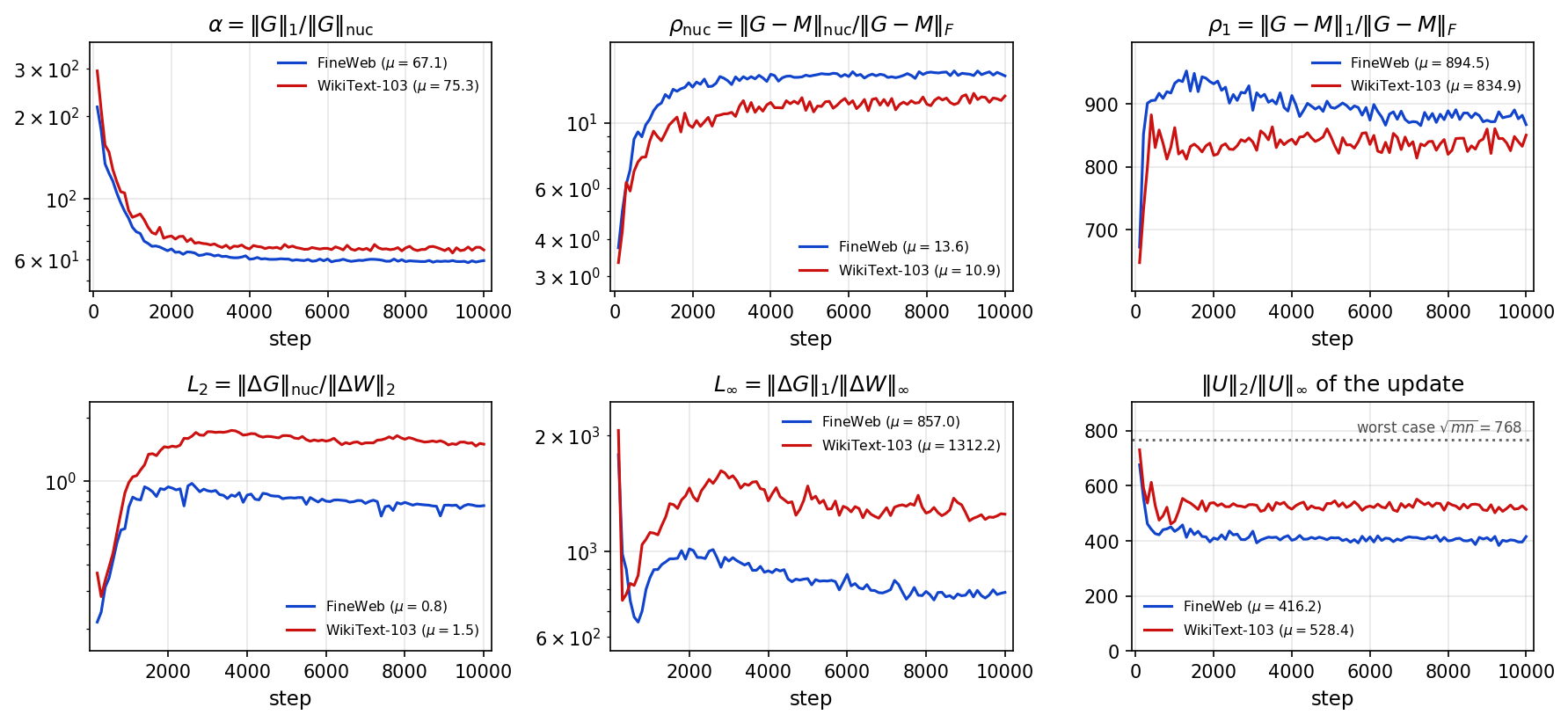}
\caption{The same constants over the course of training, on 124M runs with $P{=}2$, with the mean over the run in each legend. The gradient ratio $\alpha$ falls over the first $2000$ steps and then holds. The last panel is the update ratio $\|U\|_2/\|U\|_\infty$, which stays far below the worst case $\sqrt{mn}$ that the dense approximation would otherwise have to assume.}
\label{fig:norm-diag}
\end{figure}

Two things follow. The gradient ratio $\alpha$ falls quickly and then holds near $67$ on FineWeb and $75$ on WikiText-103, inside the range of ratios $\eta_M/\eta_L$ that the sweep prefers for \texttt{LionMuon} ($10$ to $100$, growing with $P$, Appendix~\ref{app:heatmap}). With $r_L$ well below one and $r_\rho$ about one, $\phi(P) < 1$ for every $P > 1$ on both datasets (Section~\ref{sec: theory discussion}), so the bound puts the alternating periods ahead of \texttt{Muon}. Which period is best inside the family it does not say. Figure~\ref{fig:124m} does, with a flat optimum between $P{=}2$ and $P{=}5$.

\newpage

\section{Full Experimental Setup}
\label{app:setup}

\begin{table}[h]
\centering
\caption{Special cases of \texttt{LionMuon}. Both conditions, on $\beta_1, \beta_2$ and on $P$, must hold.}
\label{tab:special-cases}
\small
\begin{tabular}{@{}lcc@{}}
\toprule
Optimizer & Momentum & Period \\
\midrule
\texttt{Signum}~\citep{bernstein2018signsgd}    & $\beta_1 = \beta_2$ & $P = \infty$ \\
\texttt{Lion}~\citep{chen2024lion}              & $\beta_1 \ne \beta_2$ (dual-EMA) & $P = \infty$ \\
\texttt{Muon}~\citep{jordan2024muon}            & $\beta_1 = \beta_2$ & $P = 1$ \\
\texttt{SignMuon} \textbf{(this work)}                 & $\beta_1 = \beta_2$ & any $P$ \\
\texttt{LionMuon} \textbf{(this work)}                    & $\beta_1 \ne \beta_2$ (dual-EMA) & any $P$ \\
\bottomrule
\end{tabular}
\end{table}

\subsection{Cost of a step}
\label{sec:cost}

Table~\ref{tab:cost} lists what a step costs on top of the forward and backward pass.

\emph{Compute.} Five Newton--Schulz iterations add about a tenth to \texttt{Muon}'s step at 124M. \texttt{LionMuon} pays them once in $P$ steps.

\emph{Memory.} \texttt{LionMuon} keeps one momentum buffer per matrix, like \texttt{Lion}, \texttt{Signum} and \texttt{Muon}, and half of \texttt{AdamW}. The token embedding and the 1D parameters, on \texttt{AdamW} under every method, add the same two buffers to all of them.

\emph{Communication.} Under data parallelism every method all-reduces the gradients, and most of that traffic hides behind the backward pass, because each bucket is sent as soon as it is ready. \texttt{Muon} then needs a second all-reduce that cannot hide: each matrix is orthogonalized on one device and the update is sent to the others~\citep{liu2025moonlight}. At 124M that is $171$\,MB per step in bf16, on top of $496$\,MB of gradients. A sign step is elementwise, so every device applies it to its own copy of the momentum and sends nothing. \texttt{LionMuon} therefore pays the second all-reduce only on \texttt{Muon} steps, $171/P$\,MB per step on average.

\texttt{Dion} and \texttt{MuonBP} cut the same cost in other ways. \texttt{Dion}~\citep{ahn2025dion} replaces the gradient all-reduce by an all-reduce of two low-rank factors with error feedback. At rank $\min(m,n)/4$ that is $114$\,MB in fp32, plus $155$\,MB of gradients for the parameters it does not factorize, mostly the token embedding: $269$\,MB in total, or $382$\,MB at rank $\min(m,n)/2$. It is the least of any method here, but all of it waits for the backward pass to finish: with the gradient all-reduce off, as in their code, even those gradients are averaged in the optimizer step. \texttt{MuonBP}~\citep{muonbp2025} orthogonalizes column blocks locally and takes a full \texttt{Muon} step every $P$-th iteration ($P{=}5$ in their paper), so it sends what \texttt{LionMuon} sends at the same $P$ but runs Newton--Schulz on every step.

\emph{Replicated \texttt{Muon}.} \texttt{Muon} can also avoid its all-reduce by orthogonalizing every matrix on every device. That trades the $171$\,MB for redundant Newton--Schulz work. At 124M on four GPUs a device then does $1.53\times10^{12}$ FLOPs of Newton--Schulz per step instead of $0.38\times10^{12}$, on top of $3.31\times10^{12}$ for forward and backward, about $31\%$ more arithmetic. \texttt{LionMuon} divides both the transfer and the Newton--Schulz work by $P$, so it is cheaper than \texttt{Muon} either way, and the loss is the same in both versions.

\begin{table}[h]
\centering \small
\caption{124M on four GPUs with data parallelism, $150{,}000$ steps, three seeds per setting; the four rate checks are single runs. Loss is the best validation loss, mean over seeds with the spread (max $-$ min) in brackets, bytes are what one step sends (Table~\ref{tab:cost}), and the last column is the first evaluation of the mean curve at or below \texttt{Muon}'s own mean final loss (evaluations every $1{,}000$ steps).}
\label{tab:ddp}
\begin{tabular}{@{}llccc@{}}
\toprule
Method & $\eta_M$ / $\eta_L$ & Loss & MB per step & Steps to \texttt{Muon}'s loss \\
\midrule
\texttt{AdamW} & $10^{-3}$ & 3.405 (0.003) & 496 & 138k \\
\texttt{Muon} & $10^{-3}$ & 3.415 (0.006) & 667 & 149k \\
\texttt{Muon} & $3\times10^{-4}$ & 3.467 & 667 & never \\
\texttt{Muon} & $3\times10^{-3}$ & 3.756 at 96k, stopped & 667 & never \\
\texttt{LionMuon} $P{=}1$ & $10^{-3}$ & 3.392 (0.002) & 667 & 135k \\
\texttt{SignMuon} $P{=}2$ & $3\times10^{-3}$ / $10^{-4}$ & 3.394 (0.003) & 581 & 137k \\
\texttt{LionMuon} $P{=}2$ & $10^{-3}$ / $10^{-4}$ & 3.373 (0.004) & 581 & 127k \\
\texttt{SignMuon} $P{=}5$ & $10^{-2}$ / $3.3\times10^{-5}$ & 3.401 (0.002) & 530 & 140k \\
\texttt{LionMuon} $P{=}5$ & $3\times10^{-3}$ / $10^{-4}$ & 3.376 (0.002) & 530 & 128k \\
\texttt{Signum} & $3\times10^{-4}$ & 3.428 (0.003) & 496 & never \\
\texttt{Lion} & $10^{-4}$ & 3.418 (0.005) & 496 & never \\
\texttt{Lion} & $3\times10^{-4}$ & spikes, 3.880 final & 496 & never \\
\texttt{MuonBP} $P{=}5$ & $10^{-3}$ & 3.411 (0.007) & 530 & 142k \\
\texttt{Dion}, rank $1/4$ & $10^{-3}$ & 3.427 (0.003) & 269 & never \\
\texttt{Dion}, rank $1/4$ & $3\times10^{-3}$ & 3.451 & 269 & never \\
\texttt{Dion}, rank $1/2$ & $10^{-3}$ & 3.410 (0.002) & 382 & 142k \\
\bottomrule
\end{tabular}
\end{table}

\begin{table}[h]
\centering \small
\caption{What it costs to reach \texttt{Muon}'s final loss, relative to \texttt{Muon}, at 124M on four GPUs. Steps are the first evaluation of the three-seed mean curve at or below \texttt{Muon}'s mean final loss, and the other columns multiply them by the per-step costs of Table~\ref{tab:cost}. Step times are given for the methods of Table~\ref{tab:cost} and for \texttt{SignMuon} $P{=}2$, whose step does the same work as that of \texttt{LionMuon} $P{=}2$. \texttt{Lion}, \texttt{Signum} and \texttt{Dion} at rank $1/4$ never reach it.}
\label{tab:reach}
\begin{tabular}{@{}lccccc@{}}
\toprule
 & Steps & FLOPs & Bytes & Time, PCIe & Time, NVLink \\
\midrule
\texttt{Muon} & \cg{0} 1.00 & \cg{0} 1.00 & \cg{0} 1.00 & \cg{0} 1.00 & \cg{44} 1.00 \\
\texttt{LionMuon} $P{=}2$ & \cg{100} \textbf{0.85} & \cg{90} 0.81 & \cg{58} 0.74 & \cg{81} 0.75 & \cg{95} 0.78 \\
\texttt{LionMuon} $P{=}5$ & \cg{93} 0.86 & \cg{100} \textbf{0.79} & \cg{71} 0.68 & \cg{97} 0.70 & \cg{100} \textbf{0.76} \\
\texttt{SignMuon} $P{=}2$ & \cg{53} 0.92 & \cg{62} 0.87 & \cg{44} 0.80 & \cg{61} 0.81 & \cg{79} 0.85 \\
\texttt{AdamW} & \cg{47} 0.93 & \cg{81} 0.83 & \cg{69} 0.69 & \cg{100} \textbf{0.69} & \cg{93} 0.79 \\
\texttt{MuonBP} $P{=}5$ & \cg{33} 0.95 & \cg{38} 0.92 & \cg{53} 0.76 & \cg{32} 0.90 & \cg{0} 1.19 \\
\texttt{Dion}, rank $1/2$ & \cg{33} 0.95 & \cg{62} 0.87 & \cg{100} \textbf{0.55} & -- & -- \\
\bottomrule
\end{tabular}
\end{table}

\begin{table}[h]
\centering
\caption{355M on FineWeb, $8.2$B tokens, every hyperparameter copied from 124M without retuning. Best validation loss, mean over three seeds with the spread (max $-$ min) in brackets.}
\label{tab:355m}
\small
\begin{tabular}{@{}lcc@{}}
\toprule
Method & $\eta_M$ / $\eta_L$ & Val loss \\
\midrule
\texttt{SignMuon} $P{=}2$ & $3\times10^{-3}$ / $10^{-4}$ & \cg{100} $\mathbf{3.009}$ ($0.002$) \\
\texttt{LionMuon} $P{=}2$ & $10^{-3}$ / $10^{-4}$ & \cg{84} $3.020$ ($0.001$) \\
\texttt{LionMuon} $P{=}5$ & $3\times10^{-3}$ / $10^{-4}$ & \cg{83} $3.021$ ($0.003$) \\
\texttt{LionMuon} $P{=}1$ & $10^{-3}$ & \cg{79} $3.024$ ($0.003$) \\
\texttt{Muon} & $10^{-3}$ & \cg{59} $3.038$ ($0.004$) \\
\texttt{Lion} & $3\times10^{-4}$ & \cg{43} $3.049$ ($0.002$) \\
\texttt{AdamW} & $10^{-3}$ & \cg{16} $3.068$ ($0.017$) \\
\texttt{SignMuon} $P{=}5$ & $10^{-2}$ / $3.3\times10^{-5}$ & \cg{10} $3.072$ ($0.001$) \\
\texttt{Signum} & $3\times10^{-4}$ & \cg{0} $3.079$ ($0.003$) \\
\bottomrule
\end{tabular}
\end{table}

\paragraph{Distributed setup.} All distributed runs use PyTorch DDP on one node with four GPUs, one process per GPU and 8 sequences of 512 tokens per GPU, the global batch of the 124M grid. The first seed of every setting and the NVLink step timings come from a node with four H200 GPUs; the PCIe timings and part of the repeated seeds come from a second node with different GPUs and interconnect (Section~\ref{sec:setup}), with the same code and settings, and those seeds agree with their first seed to within $0.006$, the size of the seed spread elsewhere in Table~\ref{tab:ddp}. Each process holds a full replica of the model and the optimizer state, and DDP averages the gradients bucket by bucket during the backward pass. What happens after it differs by method. \texttt{Muon}: the 2D parameters are dealt round-robin to the four processes, each runs Newton--Schulz on its share, and one all-reduce of the updates gives every process all of them, as in Moonlight~\citep{liu2025moonlight}. \texttt{LionMuon} and \texttt{SignMuon}: the same on \texttt{Muon} steps, nothing on sign steps. \texttt{MuonBP}: four column blocks stand in for a 4-way tensor-parallel split, the block step is local, and the full step every $P{=}5$ iterations is \texttt{Muon}'s. \texttt{Dion}: DDP's gradient all-reduce is off, each process keeps a local momentum and all-reduces two low-rank factors per matrix, plus one all-reduce of the gradients of the parameters it does not factorize, which is the data-parallel mode of their paper (their Section 3.3) and code. Algorithms~\ref{alg:muonbp} and~\ref{alg:dion} give both exactly as run. Both use \texttt{Muon}'s update scaling, so \texttt{Muon}'s learning rate transfers, and \texttt{AdamW} for the embedding and 1D parameters.

Both baselines were built for sharded models: \texttt{MuonBP} for tensor parallelism and FSDP, where \texttt{Muon} has to all-gather the shards of a matrix first (their Section 2.2), and \texttt{Dion} for FSDP and tensor parallelism, with the data-parallel sync as an option. Neither paper measures the saving end to end under data parallelism. \texttt{MuonBP} reports throughput under tensor parallelism, and \texttt{Dion} simulates step times on one GPU without communication (their Figure 1). One node under data parallelism is where we can count every byte and time every method under the same load, so that is what we use.

\begin{algorithm}{\texttt{MuonBP} for a single 2D parameter $W \in \mathbb{R}^{m \times n}$ under data parallelism}
\label{alg:muonbp}
\begin{algorithmic}[1]
\REQUIRE Period $P$, number of column blocks $B$, learning rate $\eta$, momentum $\beta$, weight decay $\lambda$, Newton--Schulz steps $K_{\mathrm{NS}}$, and $c(A) := 0.2\sqrt{\max(\text{rows}(A), \text{cols}(A))}$
\FOR{$t = 0, 1, \ldots, T-1$}
  \STATE $G_t = \nabla_W \mathcal{L}_t$ \hfill $\triangleright$ Gradient, averaged over devices by the framework
  \STATE $M_t = \beta M_{t-1} + (1-\beta)\, G_t$ \hfill $\triangleright$ Momentum, replicated on every device
  \IF{$t \bmod P = 0$}
    \STATE $U_t = c(M_t)\,\mathrm{NS}_{K_{\mathrm{NS}}}(M_t)$ \hfill $\triangleright$ Full step, each device orthogonalizes its share
    \STATE all-reduce $U_t$ \hfill $\triangleright$ The only communication of the optimizer
  \ELSE
    \STATE split $M_t$ into column blocks $M_t^{(1)}, \ldots, M_t^{(B)}$
    \STATE $U_t = \bigl[\,c(M_t^{(1)})\,\mathrm{NS}_{K_{\mathrm{NS}}}(M_t^{(1)}) \;\cdots\; c(M_t^{(B)})\,\mathrm{NS}_{K_{\mathrm{NS}}}(M_t^{(B)})\,\bigr]$ \hfill $\triangleright$ Block step, local
  \ENDIF
  \STATE $W_{t+1} = W_t - \eta\,\bigl(U_t + \lambda W_t\bigr)$
\ENDFOR
\end{algorithmic}
\end{algorithm}

\begin{algorithm}{\texttt{Dion} for a single 2D parameter $W \in \mathbb{R}^{m \times n}$ on $D$ devices}
\label{alg:dion}
\begin{algorithmic}[1]
\REQUIRE Rank $r$, error-feedback rate $\gamma$, learning rate $\eta$, weight decay $\lambda$, $M^{(d)}_{-1} = 0$ on every device $d$, $V_{-1} \in \mathbb{R}^{n \times r}$ random with orthonormal columns and the same on every device, and $c := 0.2\sqrt{\max(m,n)}$
\FOR{$t = 0, 1, \ldots, T-1$}
  \STATE $G^{(d)}_t = \nabla_W \mathcal{L}^{(d)}_t$ \hfill $\triangleright$ Local gradient, the framework's gradient all-reduce is off
  \STATE $M^{(d)}_t = M^{(d)}_{t-1} + G^{(d)}_t$ \hfill $\triangleright$ Local momentum, accumulated without decay
  \STATE $P_t = \tfrac{1}{D}\sum_{d} M^{(d)}_t V_{t-1}$ \hfill $\triangleright$ all-reduce of an $m \times r$ factor
  \STATE $U_t = \mathrm{QR}(P_t)$ \hfill $\triangleright$ Orthonormal columns, one power iteration
  \STATE $R^{(d)}_t = M^{(d)\top}_t U_t$, \quad $R_t = \tfrac{1}{D}\sum_{d} R^{(d)}_t$ \hfill $\triangleright$ all-reduce of an $n \times r$ factor
  \STATE $M^{(d)}_t \leftarrow M^{(d)}_t - \gamma\, U_t R^{(d)\top}_t$ \hfill $\triangleright$ Error feedback, local
  \STATE $V_t = R_t\,\mathrm{diag}\bigl(\|R_t e_1\|_2, \ldots, \|R_t e_r\|_2\bigr)^{-1}$ \hfill $\triangleright$ Columns normalized
  \STATE $W_{t+1} = W_t - \eta\,\bigl(c\, U_t V_t^\top + \lambda W_t\bigr)$
\ENDFOR
\end{algorithmic}
\end{algorithm}

\begin{table}[ht!]
\centering \small
\caption{Full experimental configuration.}
\label{tab:setup}
\small
\begin{tabular}{@{}lp{0.6\linewidth}@{}}
\toprule
\multicolumn{2}{l}{\textit{Model architecture}} \\
\midrule
Number of layers      & 12 \\
Number of heads       & 12 \\
Embedding dim         & 768 \\
Sequence length       & 512 \\
Vocabulary size       & 50{,}304 (GPT-2 BPE) \\
Architectures         & GPT-2 base \\
\midrule
\multicolumn{2}{l}{\textit{Training schedule}} \\
\midrule
Iterations            & 64{,}000 \\
Warmup steps          & 3{,}000 \\
Batch size            & 32 \\
Gradient accumulation & 1 \\
LR scheduler          & cosine \\
Weight decay          & 0.1 \\
Gradient clipping     & 0.5 \\
Eval interval         & every 500 steps \\
\midrule
\multicolumn{2}{l}{\textit{Optimizer-specific}} \\
\midrule
Newton--Schulz steps                  & $K_{\mathrm{NS}} = 5$ \\
NS scaling                           & $0.2\sqrt{\max(m,n)}$ \\
\texttt{AdamW} $(\beta_1, \beta_2)$           & $(0.8, 0.999)$ \\
\texttt{Lion} $(\beta_1, \beta_2)$            & $(0.9, 0.99)$ \\
\texttt{Signum} momentum                      & $\beta = 0.9$ (EMA) \\
\texttt{Muon} momentum                        & $\beta = 0.9$ (EMA), no Nesterov \\
\texttt{SignMuon} momentum                    & $\beta = 0.9$ (EMA), no Nesterov \\
\texttt{LionMuon} $(\beta_1, \beta_2)$          & $(0.9, 0.99)$ \\
Embedding and 1D parameters          & \texttt{AdamW}, $\eta = 10^{-3}$, betas $(0.8, 0.999)$ \\
\midrule
\multicolumn{2}{l}{\textit{Distributed run (Section~\ref{sec:results-distributed})}} \\
\midrule
Framework             & PyTorch DDP, NCCL, one process per GPU \\
Hardware              & one node, $4\times$ H200 (NVLink); part of the seed repeats on a second node \\
Per-GPU batch         & 8 sequences $\times$ 512 tokens (global batch 32) \\
Iterations / warmup   & 150{,}000 / 7{,}500 \\
Eval                  & every 1{,}000 steps, 32 batches \\
\texttt{MuonBP}       & 4 column blocks, full step every 5, momentum 0.9 \\
\texttt{Dion}         & rank $\min(m,n)/4$, error feedback 0.05, gradient all-reduce off \\
\bottomrule
\end{tabular}
\end{table}

\section{Hyperparameter Tuning}
\label{app:heatmap}

Every number in the main text comes from a learning rate chosen on the \emph{full} training horizon: each cell of the sweep is a complete $64{,}000$-step run of the same model, data and schedule as the headline runs (12 layers, width 768, batch $32\times512$ tokens, cosine schedule, evaluation every $500$ steps), and a cell's score is the best validation loss it reaches. We tune two quantities: the spectral-step learning rate $\eta_M$ and the ratio $\alpha = \eta_M/\eta_L$ to the sign-step learning rate, which is the pair the theory in Section~\ref{sec: theory discussion} identifies. For pure \texttt{Muon} ($P{=}1$) and \texttt{Signum} ($P{=}\infty$) there is only one rate to sweep.

\paragraph{Grid and pruning.} $\eta_M$ is swept over $10^{-4}$ to $10^{-1}$ and $\alpha$ over $1$ to $3000$, both in half-decade steps, with the swept band following the optimum as it moves with $P$. Running every pair at full length would be wasteful, so a cell is stopped early once it is clearly out of the running: from $34$ finished curves, no eventual winner was ever more than $0.10$ behind the best cell of its own $(P, \beta)$ group at $6{,}000$ steps or $0.20$ behind at $16{,}000$, and we stop cells that exceed those margins. Tables~\ref{tab:grid-fw} and~\ref{tab:grid-wt} report how many cells ran to completion and how many were stopped this way; a stopped neighbour is written ``cut''.

\paragraph{Selected cells.} For each method we take the best cell and use it verbatim: for the three-seed runs at 124M, and, without any further tuning, for 355M. The tables give the chosen pair together with the losses of the neighbouring $\eta_M$ values, which is what makes the optimum interior rather than an edge of the grid. Two observations transfer to practice. The ratio $\alpha$ grows with the period ($10$, $30$, $100$ at $P{=}2, 5, 20$ for \texttt{LionMuon} on both datasets) while the sign-step rate stays in a narrow band ($\eta_L \approx 10^{-4}$ on FineWeb, $3\times10^{-4}$ on WikiText-103), so it is $\eta_M$ that moves with $P$: a practitioner who changes $P$ should keep $\eta_L$ and scale $\eta_M$ with the period.

\paragraph{Why not a cheap pilot.} We first chose learning rates on $3{,}000$-step pilot runs and transferred them. We found that this mis-ranks the candidates: the cell a pilot prefers is not the cell that wins at $64{,}000$ steps, and the gap between the pilot's choice and the full-horizon choice is larger than the differences we report between methods. Every tuned number in this paper therefore comes from the full-horizon sweep; the pilot-tuned grid, over three datasets and two architectures, gave the same ranking of methods, but we do not report it here.

\begin{table}[ht]
\centering
\caption{The tuning sweep at 124M on FineWeb, every finished cell a full $64{,}000$-step run. ``Cells'' counts the runs that went to the end plus the ones the pruning rule stopped early. The last column is the loss of the chosen cell, with the losses at the next lower and next higher $\eta_M$ (same $\alpha$) in brackets.}
\label{tab:grid-fw}
\small
\begin{tabular}{llcccc}
\toprule
method & $\eta_M$ swept & $\alpha$ swept & cells & chosen $(\eta_M,\alpha)$ & loss (next $\eta_M$ down / up) \\
\midrule
\texttt{Muon} ($P{=}1$) & 0.0003--0.003 & -- & 3+0 & 0.001 & \textbf{3.528} (3.596 / 3.534) \\
\texttt{SignMuon} $P{=}2$ & 0.001--0.01 & 3--300 & 8+3 & (0.003, 30) & \textbf{3.511} (3.529 / cut) \\
\texttt{SignMuon} $P{=}5$ & 0.001--0.03 & 3--1000 & 13+3 & (0.01, 300) & \textbf{3.508} (3.533 / cut) \\
\texttt{SignMuon} $P{=}20$ & 0.003--0.1 & 10--1000 & 8+4 & (0.01, 100) & \textbf{3.540} (3.581 / 3.551) \\
\texttt{LionMuon} $P{=}1$ & 0.0003--0.003 & -- & 2+1 & 0.001 & \textbf{3.511} (3.540 / cut) \\
\texttt{LionMuon} $P{=}2$ & 0.0003--0.01 & 1--300 & 12+2 & (0.001, 10) & \textbf{3.501} (3.537 / 3.507) \\
\texttt{LionMuon} $P{=}5$ & 0.001--0.01 & 3--300 & 9+2 & (0.003, 30) & \textbf{3.500} (3.523 / cut) \\
\texttt{LionMuon} $P{=}20$ & 0.001--0.03 & 10--300 & 8+3 & (0.01, 100) & \textbf{3.511} (3.539 / cut) \\
\texttt{AdamW} & 0.0003--0.003 & -- & 3+0 & 0.001 & \textbf{3.548} (3.581 / 3.577) \\
\texttt{Lion} & 0.0001--0.001 & -- & 2+1 & 0.0003 & \textbf{3.567} (3.600 / cut) \\
\texttt{Signum} & 0.0001--0.001 & -- & 2+1 & 0.0003 & \textbf{3.590} (3.600 / cut) \\
\bottomrule
\end{tabular}
\end{table}

\begin{table}[ht]
\centering
\caption{The same sweep on WikiText-103.}
\label{tab:grid-wt}
\small
\begin{tabular}{llcccc}
\toprule
method & $\eta_M$ swept & $\alpha$ swept & cells & chosen $(\eta_M,\alpha)$ & loss (next $\eta_M$ down / up) \\
\midrule
\texttt{Muon} ($P{=}1$) & 0.001--0.01 & -- & 2+1 & 0.003 & \textbf{2.838} (2.884 / cut) \\
\texttt{SignMuon} $P{=}2$ & 0.0003--0.03 & 1--3000 & 12+5 & (0.003, 3) & \textbf{2.835} (2.885 / cut) \\
\texttt{SignMuon} $P{=}5$ & 0.0001--0.03 & 1--100 & 10+5 & (0.01, 30) & \textbf{2.842} (2.890 / cut) \\
\texttt{SignMuon} $P{=}20$ & 0.001--0.1 & 3--3000 & 9+8 & (0.03, 100) & \textbf{2.859} (2.902 / cut) \\
\texttt{LionMuon} $P{=}1$ & 0.001--0.01 & -- & 2+1 & 0.003 & \textbf{2.826} (2.861 / cut) \\
\texttt{LionMuon} $P{=}2$ & 0.0003--0.01 & 1--3000 & 8+8 & (0.003, 10) & \textbf{2.827} (2.881 / cut) \\
\texttt{LionMuon} $P{=}5$ & 0.0003--0.03 & 1--100 & 8+5 & (0.01, 30) & \textbf{2.825} (2.872 / cut) \\
\texttt{LionMuon} $P{=}20$ & 0.001--0.1 & 3--300 & 9+5 & (0.03, 100) & \textbf{2.839} (2.877 / cut) \\
\texttt{AdamW} & 0.0003--0.01 & -- & 4+0 & 0.003 & \textbf{2.868} (2.913 / 2.944) \\
\texttt{Lion} & 0.0001--0.001 & -- & 3+0 & 0.0003 & \textbf{2.889} (2.907 / 4.407) \\
\texttt{Signum} & 0.0001--0.003 & -- & 4+0 & 0.001 & \textbf{2.895} (2.914 / 5.373) \\
\bottomrule
\end{tabular}
\end{table}

\begin{table}[ht]
\centering
\caption{The tuned 124M runs with three seeds: best validation loss after $64{,}000$ steps, mean over the seeds with the spread (max $-$ min) in brackets. These are the values Figure~\ref{fig:124m} plots.}
\label{tab:124m-seeds}
\small
\begin{tabular}{@{}lcc@{}}
\toprule
 & FineWeb & WikiText-103 \\
\midrule
\texttt{AdamW} & $3.553$ ($0.017$) & $2.869$ ($0.003$) \\
\texttt{Lion} & $3.566$ ($0.005$) & $2.893$ ($0.008$) \\
\texttt{Signum} & $3.593$ ($0.016$) & $2.896$ ($0.002$) \\
\texttt{Muon} & $3.527$ ($0.001$) & $2.835$ ($0.004$) \\
\texttt{LionMuon} $P{=}1$ & $3.510$ ($0.004$) & $2.826$ ($0.001$) \\
\texttt{LionMuon} $P{=}2$ & $3.498$ ($0.004$) & $2.827$ ($0.002$) \\
\texttt{LionMuon} $P{=}5$ & $3.498$ ($0.003$) & $2.825$ ($0.002$) \\
\texttt{LionMuon} $P{=}20$ & $3.510$ ($0.002$) & $2.840$ ($0.003$) \\
\texttt{SignMuon} $P{=}2$ & $3.509$ ($0.003$) & $2.839$ ($0.011$) \\
\texttt{SignMuon} $P{=}5$ & $3.507$ ($0.004$) & $2.841$ ($0.004$) \\
\texttt{SignMuon} $P{=}20$ & $3.539$ ($0.004$) & $2.859$ ($0.000$) \\
\bottomrule
\end{tabular}
\end{table}

\subsection{Momentum}
\label{app:betas}

The momentum pair was tuned on FineWeb at 124M with $P{=}2$. For each $(\beta_1,\beta_2)$ the rate $\eta_M$ and the ratio $\alpha$ were swept as above, and Table~\ref{tab:betas} reports the best cell. The diagonal $\beta_1=\beta_2$ is \texttt{SignMuon}, the rest is \texttt{LionMuon}. Every off-diagonal pair beats every diagonal one, and $(0.9,0.99)$ is best, so \texttt{LionMuon} and \texttt{Lion} use it. The diagonal is flat within $0.002$, so \texttt{SignMuon} and \texttt{Muon} keep the usual $0.9$. \texttt{AdamW} uses $(0.8,0.999)$, the pair tuned by the benchmark, which also beat $(0.9,0.95)$ in our check.

\begin{table}[h]
\centering
\caption{Loss at 124M on FineWeb with $P{=}2$ for each momentum pair, after tuning the rate and the ratio for that pair. The diagonal is \texttt{SignMuon}, the rest is \texttt{LionMuon}.}
\label{tab:betas}
\small
\begin{tabular}{@{}lccc@{}}
\toprule
$\beta_1 \backslash \beta_2$ & $0.9$ & $0.95$ & $0.99$ \\
\midrule
$0.9$ & \cg{0} $3.511$ & \cg{30} $3.508$ & \cg{100} $\mathbf{3.501}$ \\
$0.95$ & & \cg{20} $3.509$ & \cg{90} $3.502$ \\
$0.99$ & & & \cg{10} $3.510$ \\
\bottomrule
\end{tabular}
\end{table}

\subsection{Two ablations: a random period, and \texttt{Lion} on the embedding and 1D parameters}
\label{app:ablate}

Two questions are natural. Does the period have to be fixed, and do the parameters outside the matrices, the token embedding and the 1D parameters, need \texttt{AdamW}? For the first we replace the fixed period by a \texttt{Muon} step drawn with probability $1/P$ at every iteration, so the mean spacing is the same but the steps are irregular. For the second we give the embedding, the biases and the norm gains the same \texttt{Lion} step as the matrices take on sign steps, at the sign-step rate and betas, instead of \texttt{AdamW}. Both run at 124M on FineWeb for $64{,}000$ steps at the tuned rates of Appendix~\ref{app:heatmap}, with three seeds (Table~\ref{tab:ablate}). The random period changes nothing. \texttt{Lion} on the parameters outside the matrices is clearly worse, by about $0.04$, so the \texttt{AdamW} fallback is worth its second buffer there.

\begin{table}[h]
\centering
\caption{Loss at 124M on FineWeb, $64{,}000$ steps, mean over three seeds with the spread (max $-$ min) in brackets. The first column is the tuned setting of Section~\ref{sec:results}, the other two are the ablations.}
\label{tab:ablate}
\small
\begin{tabular}{@{}lccc@{}}
\toprule
 & Fixed period & Random period & \texttt{Lion} outside the matrices \\
\midrule
\texttt{LionMuon} $P{=}2$ & $3.498$ ($0.004$) & $3.499$ ($0.003$) & $3.535$ ($0.001$) \\
\texttt{LionMuon} $P{=}5$ & $3.498$ ($0.003$) & $3.499$ ($0.003$) & $3.538$ ($0.002$) \\
\bottomrule
\end{tabular}
\end{table}

\section{Heavy-ball versus EMA momentum}
\label{app:hb-ema}

For \texttt{Muon} and \texttt{SignMuon} the momentum is in EMA form, $M_t = \beta M_{t-1} + (1-\beta)\,G_t$; the heavy-ball form $M_t = \beta M_{t-1} + G_t$ of the \texttt{llm-baselines} codebase \citep{semenov2025benchmark} gives the same buffer up to the constant factor $1-\beta$ and, since $\mathrm{sign}$ and $\msign$ are scale-free, the same iterates at the same learning rate.
\end{appendixpart}
\end{document}